\documentclass{article}
\usepackage[utf8]{inputenc}
\usepackage[a4paper, total={6in, 8in}]{geometry}
\usepackage{authblk}
\usepackage{amsthm}
\usepackage{amsmath}
\usepackage{amssymb}
\usepackage{graphicx}
\usepackage{bbm}
\usepackage{braket,amsfonts,amssymb}
\usepackage{comment}
\usepackage{amsopn}
\usepackage{bm}
\usepackage{mathrsfs}  
\usepackage[mathscr]{euscript}
\usepackage{bm,nicefrac}
\usepackage{extarrows}
\usepackage{caption}
\usepackage{subcaption}
\usepackage{xcolor}
\usepackage[linktoc=all]{hyperref}
\usepackage{titlesec}
\hypersetup{
    colorlinks,
    citecolor=blue,
    filecolor=black,
    linkcolor=blue,
    urlcolor=black
}

\definecolor{WNconstant}{RGB}{204, 88, 3}
\def\a{{a}}

\def\b{{\bf b}}

\def\bo{{\bf 1}}
\def\c{c}

\def\g{r}
\def\ginfty{\g_\infty}
\def\gprod{\widetilde{\g}}
\def\h{h}

\def\j{j}
\def\k{k}
\def\lrr{{\tilde{\eta}}}
\def\m{m}
\def\n{n}

\def\k{{\bf k}}

\def\r{\gamma}

\def\s{s}
\def\t{t}
\def\u{{\bf u}}

\def\v{{\bf v}}

\def\w{{\bf u}}
\def\we{u}
\def\ww{{\bf w}}
\def\wwe{w}
\def\x{{\bf x}}
\def\xe{x}
\def\y{{\bf b}}

\def\z{{\bf z}}

\def\wprod{\widetilde{\w}}

\def\xwn{\x_\text{wn}}
\def\xwnprod{\tilde{\x}_\text{wn}}
\def\xprod{\tilde{\x}}
\def\xprode{\tilde{\xe}}

\def\xprodinfty{\xprod_\infty}

\def\A{{\bf A}}
\def\Ae{A}

\def\F{F}

\def\I{I}

\def\K{K}
\def\L{L}
\def\Loss{\mathcal{L}}
\def\LossM{\widetilde{\mathcal{L}}}

\def\M{M}
\def\N{N}

\def\Pw{\mathcal{P}_{\w}}
\def\PA{\mathcal{P}_{\A}}
\def\Px{\mathcal{P}_{\x}}

\def\Q{{Q}}
\def\R{\mathbb{R}}

\def\T{T}

\def\0{\boldsymbol{0}}
\def\1{\boldsymbol{1}}

\def\tT{\top}

\DeclareMathOperator*{\argmin}{arg\,min}

\renewcommand{\epsilon}{\varepsilon}

\renewcommand{\L}{L}

\providecommand{\keywords}[1]{\textbf{\textit{Keywords --- }} #1}

\newtheorem*{theorem*}{Theorem}
\newtheorem{theorem}{Theorem}[section]

\newtheorem{definition}[theorem]{Definition}
\newtheorem{lemma}[theorem]{Lemma}

\theoremstyle{remark}
\newtheorem{remark}[theorem]{Remark}

\numberwithin{equation}{section}

\title{Robust Implicit Regularization via Weight Normalization}

\begin{document}

\author[1]{Hung-Hsu Chou}
\author[2]{Holger Rauhut}
\author[3]{Rachel Ward}
\affil[1]{School of Computation, Information and Technology, Technical University of Munich, Germany}
\affil[2]{Mathematics Institute, Ludwig Maximilian University of Munich, Germany}
\affil[3]{Oden Institute for Computational Engineering \& Sciences, University of Texas at Austin, USA}

\maketitle

\begin{abstract}
Overparameterized models may have many interpolating solutions; implicit regularization refers to the hidden preference of a particular optimization method towards a certain interpolating solution among the many. A by now established line of work has shown that (stochastic) gradient descent tends to have an implicit bias towards low rank and/or sparse solutions when used to train deep linear networks, explaining to some extent why overparameterized neural network models trained by gradient descent tend to have good generalization performance in practice.
However, existing theory for square-loss objectives often requires very small initialization of the trainable weights, which is at odds with the larger scale at which weights are initialized in practice for faster convergence and better generalization performance.  In this paper, we aim to close this gap by incorporating and analyzing gradient flow (continuous-time version of gradient descent) with \emph{weight normalization}, where the weight vector is reparameterized in terms of polar coordinates, and gradient flow is applied to the polar coordinates.  By analyzing key invariants of the gradient flow and using Lojasiewicz’s Theorem, we show that weight normalization also has an implicit bias towards sparse solutions in the diagonal linear model, but that in contrast to plain gradient flow, weight normalization enables a robust bias that persists even if the weights are initialized at practically large scale.  Experiments suggest that the gains in both convergence speed and robustness of the implicit bias are improved dramatically by using weight normalization in overparameterized diagonal linear network models.
\end{abstract}

\keywords{implicit regularization, weight normalization, gradient descent, overparameterization, linear neural network, vector factorization, L1 minimization, compressed sensing}


\section{Introduction}
\label{sec:Introduction}

Unlike many classical models such as linear regression or kernel methods, recent machine learning breakthroughs are often based on overparameterized models, e.g. neural networks, where the number of data is less than the number of parameters. To develop theoretical understanding of modern machine learning, many researchers focus on analyzing the simplified model, the linear network \cite{arora2018optimization,cohen2023deep,Bach2019implicit,gunasekar2019implicit,neyshabur2017geometry,neyshabur2015}, where the activation function is the identity. From those studies, a phenomenon known as implicit regularization gradually emerges from the fog.

Implicit regularization refers to the hidden preference of the learning model, in contrast to explicit regularization which is explicitly specified in the training process. In particular, implicit regularization can be found in vector \cite{soudry2018implicit}, matrix \cite{arora2019implicit,chou2020implicit,geyer2019implicit,gunasekar2017implicit,neyshabur2017geometry,neyshabur2015,razin2020implicit,stoger2021small}, and tensor \cite{Razin2021implicitTensor,Razin2022hierachicalTensor} factorization. In these examples, we understand theoretically that gradient descent (GD) applied to simple overparameterized models exhibits implicit regularization for sparse/low-rank solutions, i.e. solutions of low complexity. Therefore, in applications where low complexity is desirable, the algorithm is guaranteed to perform well.

Yet, many of these theoretical guarantees only hold for GD with small \cite{chou2021more,stoger2021small} or infinitesimal \cite{arora2018optimization,arora2019implicit} initialization, which is not practical because small initialization leads to slow convergence -- in fact, as initialization decreases, the time required to converge to a small neighborhood of zero increases. In practice, GD is initialized very differently. For example, a common setting for neural networks is the Xavier initialization \cite{glorot2010understand}, where the initial weights are normalized independent Gaussian vectors. Such scaling leads to not only empirical success but also is theoretically justified by the neural tangent kernel \cite{du2019gradient,Jacot2018ntk}.

This gap between implicit bias theory (which requires small initialization) and practice (where initialization is often not small but normalized) indicates that the algorithm which has so far been the main focus of study for implicit bias  -- (stochastic) gradient descent, or (S)GD -- might be too simplistic compared to the algorithms used in practice to train neural networks. As the authors in \cite{Gunasekar18geometry,soudry2018implicit} point out, the theoretical limitation might be due to the choice of loss function, for instance the commonly used $\ell_2$ loss. It was shown that GD on loss functions with exponential tails, such as exponential, logistic, and sigmoid losses, in general does not require small initialization. However, the optimization procedure requires certain notions of normalization, otherwise the iterates are likely to blow up. Hence it is natural to consider combining normalization with $\ell_2$ loss to remove the constraints on initialization.

Indeed, \emph{normalization} of some form is an important modification to plain (S)GD used in practice for accelerating convergence and generalization. Batch normalization \cite{ioffe2015batch} and layer normalization \cite{ba2016layer} are among the most popular choices, while weight normalization (WN) \cite{Salimans2016weight} was one of the earliest proposed normalization algorithms and represents a simple model for batch normalization. 

In \cite{Wu2019implicit}, GD with WN was shown to induce implicit bias towards the minimal $\ell_2$-norm solution in a region of initialization in the setting of overparameterized linear regression. The impact of WN on implicit bias in the linear regression setting hints that WN might be a fundamental algorithmic aspect of the implicit bias observed in practice. 

Weight normalization re-parameterizes the weight vector in each layer in polar coordinates,
\begin{align}
    \label{eq:polar}
    \x = \frac{\g}{\|\w\|_2}\w;
\end{align}
(S)GD is then implemented separately with respect to the vector $\w$ and magnitude scalar $\g$. While overwhelming empirical evidence shows that weight normalization induces faster convergence of (S)GD towards solutions with better generalization performance, rigorous theoretical proofs of these effects have remained challenging due to the nonlinearity introduced in \eqref{eq:polar}. 

It is then natural to study GD with WN, and in particular to hope that WN can induce a more robust implicit bias towards low-complexity solutions. To be concise, we will be analyzing gradient flow (GF), the continuous-time version of gradient descent, in the following context. There are works establishing connections between GD and GF, e.g. \cite{chou2020implicit,Nguegnang2021Convergence}, which show that for sufficiently small step-size, GD and GF exhibit similar behaviours. 

\subsection{Our contribution and related work}

In this paper, we show that GF with WN, when applied to the standard diagonal linear model for vector/matrix factorization, achieves implicit bias/regularization towards sparse solutions without small initialization. In short, we show that
\begin{center}
 \emph{Weight normalization provably induces a robust implicit bias/regularization.}
\end{center}
The implicit bias/regularization is robust in the sense that it does not depend on the initialization as much as many other works suggested \cite{arora2018optimization,arora2019implicit,chou2021more,stoger2021small}.

Previous papers have analyzed implicit bias induced by normalization.  In the context of classification using multilayer linear neural networks, \cite{morwani2022inductive} and \cite{poggio2020complexity} analyzed gradient flow with WN. These papers do not study the relationship between robustness of initialization and implicit bias because in the context of classification, even plain GF without normalization exhibits an implicit bias to max-margin (min $\ell_2$-norm) solutions independent of initialization.  

In \cite{wu2018wngrad} the authors established a connection between adaptive GF and WN and provided robust convergence guarantees for weight-normalized GF. \cite{ioffe2015batch} extended these convergence guarantees to batch normalization.  The papers \cite{dukler2020optimization,wu2022adaloss} provided linear convergence of normalized GF methods in the setting of multilayer ReLU networks in the neural tangent kernel regime.

\cite{Wu2019implicit} showed that gradient flow with respect to WN induces robust implicit bias towards the minimal $\ell_2$ norm solution, in the radius $\g_0 < \g^{*}$ where $\g^{*}$ is the magnitude of the minimal $\ell_2$ norm solution. 

Our work extends the results from \cite{Wu2019implicit}, proving that \emph{WN induces robust implicit bias in a family of overparameterized diagonal linear network models of arbitrary depth}, which includes overparameterized least squares as the base model.
The loss corresponding to such family takes the form
\begin{equation}\label{eq:overloss_1}
    \Big\|\A\big(\x^{(1)}\odot \cdots\odot \x^{(\L)}\big)-\y\Big\|_2^2
\end{equation}
where $\odot$ is the entry-wise product. It is shown in \cite{chou2021more} that GF on \eqref{eq:overloss_1} under identical initialization is equivalent to GF on
\begin{equation}\label{eq:overloss_2}
    \Big\|\A\x^{\odot\L}-\y\Big\|_2^2.
\end{equation}On the other hand, \cite{Wu2019implicit} showed that applying WN \eqref{eq:polar} to overparameterized linear regression induces implicit regularization towards minimal $\ell_2$ norm, that is, GF on the loss 
\begin{align}\label{eq:L2}
    \Big\|\A\Big(\frac{\g}{\|\w\|_2}\w\Big) - \y\Big\|_2^2
\end{align}
converges to the limit such that
\begin{equation}\label{eq:weight_normalization_original}
    \lim_{\t\to\infty}\g(\t)\w(\t) \approx \argmin_{\A\z=\y}\|\z\|_2.
\end{equation}

In this paper we generalize the proof strategy in \cite{Wu2019implicit} to obtain a robust $\ell_1$-minimization solver, precisely by showing that for $\L\geq 2$, GF on the loss function (where $\odot$ denotes the entry-wise product/power)
\begin{align}\label{eq:L3}
    \Big\|\A\Big(\frac{\g}{\|\w\|_2}\w\Big)^{\odot\L} - \y\Big\|_2^2
\end{align}
converges to the limit such that
\begin{equation}
    \lim_{\t\to\infty}(\g(\t)\w(\t))^{\odot\L} \approx \argmin_{\A\z=\y}\|\z\|_1.
\end{equation}
Although the implicit bias towards minimal $\ell_1$-norm solution has been studied in many works , e.g. \cite{Bach2019implicit,pesme2023saddletosaddle}, \emph{our method does not necessarily require small initialization}, which opens the possibility in understanding networks trained with larger initialization.

\begin{figure}[h]
    \centering
    \begin{subfigure}[b]{0.47\textwidth}
        \centering
        \includegraphics[width=\textwidth]{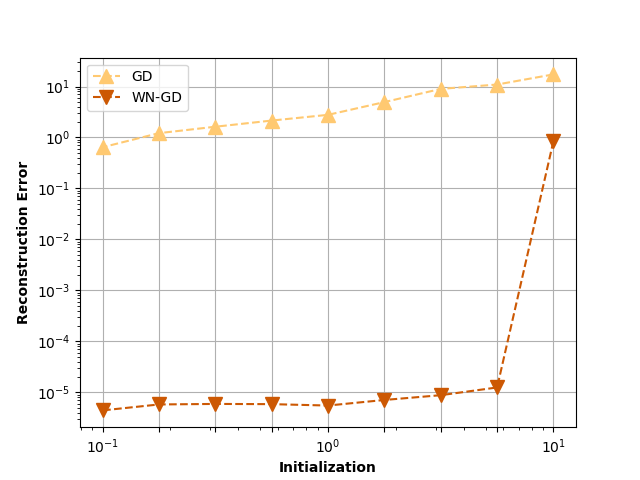}
        \caption{Gradient descent with weight normalization (WN-GD) yields significantly smaller reconstruction error.}
        \label{fig:WN_TestError_Initialization_0}
    \end{subfigure}
    \qquad\qquad
    \begin{subfigure}[b]{0.4\textwidth}
        \centering
        \includegraphics[width=\textwidth]{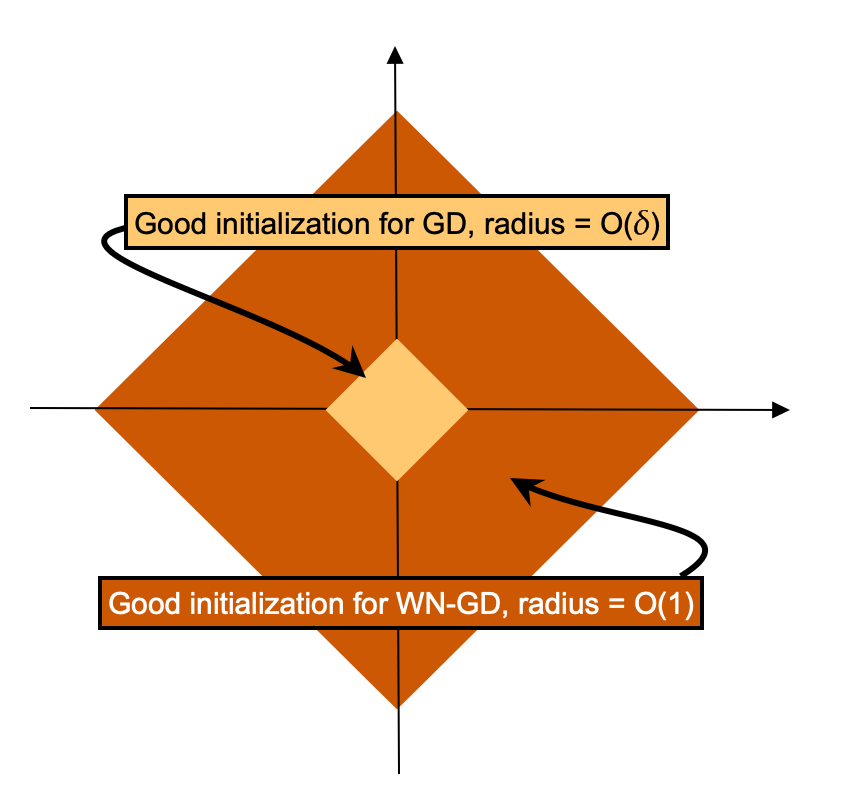}
        \caption{For a fixed error upper bound, WN-GD allows a much wider range of initialization. Here $\delta$ is given in \eqref{eq:delta_ball}.}
        \label{fig:wn_train_error_dynamic_0}
    \end{subfigure}
    \caption{WN converges to minimal $\ell_1$-norm solutions from a wider range of initialization, and hence is more robust in the sense that it is not sensitive to the choice of initialization. This suggests that GF with WN could be used as an efficient alternative for $\ell_1$-minimization.  Each data point in Figure \ref{fig:WN_TestError_Initialization_0} is an average over multiple random initializations at fixed scale. The improvement ratio (reconstruction error for GD divided by reconstruction error for WN-GD) can be huge; when the initialization scale $\alpha$ (defined in \eqref{eq:alpha}) equals to $0.1$, such ratio is more than $10^5$!}
    \label{fig:wn_compare_0}
\end{figure}

Our work also builds on a long line of work on the factorized linear multilayer model \cite{hoff2017lasso,li2021implicit,vaskevicius2019implicit,woodworth2020kernel,You2020RobustRV,zhao2019implicit}. To precisely quantify the error coming from the approximation ``$\approx$'' in \eqref{eq:weight_normalization_original}, we seek inspiration from \cite{chou2021more}, where the authors showed that for $\L\geq 2$, GF on the loss \eqref{eq:overloss_2} converges to the limit whose $\ell_1$-norm is almost minimized in the following sense. Let $\xprod = \x^{\odot\L}$. Then the limit $\xprodinfty:=\lim_{\t\to\infty}\xprod(\t)$ satisfies $\A\xprodinfty=\b$ and
\begin{equation}\label{eq:alpha_bound_original}
    \|\xprodinfty\|_1 - \min_{\A\z=\y} \|\z\|_1\leq \epsilon(\|\xprod_0\|_1) \cdot\min_{\A\z=\y} \|\z\|_1
\end{equation}
where $\xprod_0=\xprod(0)$ and $\epsilon:\R\to\R$ is an increasing function (which we will later used and specified) with $\epsilon(0)=0$. In general, we need $\|\xprod_0\|_1\ll 1$ to have $\epsilon(\|\xprod_0\|_1)\ll 1$. This often implies that if the initialization is on the unit sphere, the error $\epsilon$ is simply too large for \eqref{eq:alpha_bound_original} to be useful at all.

Fortunately, this problem disappears when we incorporate normalization as in \eqref{eq:L3}. In Theorem \ref{theorem:optimality_constant_rate}, we show that the error bound can be improved to
\begin{equation}\label{eq:alpha_bound_wn}
    \|\xprodinfty\|_1 - \min_{\A\z=\y} \|\z\|_1\leq \epsilon(\rho^{-\L}\|\xprod_0\|_1) \cdot\min_{\A\z=\y} \|\z\|_1
\end{equation}
for some $\rho\geq1$ that increases as $\x_0$ decreases (the explicit dependence will be later stated in \eqref{eq:rho}). As a consequence, the right hand side remains small even for moderately small initialization such as the normalized Gaussian vector, which often leads to smaller generalization error \cite{du2019gradient,Jacot2018ntk}. 
%

\subsection{Formulation}
\label{subsection:steup}

We now introduce the two main models, gradient flow without normalization and gradient flow with normalization, and a key parameter $\lrr$, which we call the learning rate ratio. They are defined as follows. Let $\A\in\mathbb{R}^{\M\times\N}$ and $\y\in\mathbb{R}^{\M}$ be given. Let $\L\in\mathbb{N}$ and consider the following loss function:
\begin{align}
    \label{eq:factorized_loss}
    \Loss(\x) &= \frac{1}{2\L}\|\A\x^{\odot\L} - \y\|_2^2.
\end{align}
We say that $\x(\t)$ follows the gradient flow {\bf(without WN)} if
\begin{align}
    \label{eq:dxdt}
    \partial_{\t}\x&= - \nabla\Loss(\x),\quad\x(0) = \x_0.
\end{align}
This is the setting of many previous works on implicit bias in overparameterized models. In this paper, we define the loss function with WN as
\begin{align}
    \label{eq:factorized_loss_weight_normalization}
    \LossM(\g,\w) &= \Loss\left(\frac{\g}{\|\w\|_2}\w\right)
\end{align}
for $\g\in\mathbb{R}$ and $\w\in\mathbb{R}^{\N}$. We say that $(\g(\t),\w(\t))$ follows the gradient flow {\bf with WN} if
\begin{align}
    \partial_{\t}\g&= - \eta_\g\nabla_{\g}\LossM(\g,\w),\quad\g(0) = \g_0\label{eq:dgdt}\\
    \partial_{\t}\w&= - \eta_\w\nabla_{\w}\LossM(\g,\w),\quad\w(0) = \w_0\label{eq:dwdt}
\end{align}
where $\eta_\g,\eta_\w > 0$ are the learning rates for the respective parameters. We can always assume $\eta_\w=1$ without loss of generality (by Lemma \ref{lemma:rescaled_learning_rate} below). Unless otherwise specified, the default setting in this paper is that $\eta_\g$ equals to some positive time-independent constant $\lrr$. We call this constant the {\bf learning rate ratio}. The choice of $\lrr$ is important and will be discussed later on in the remarks after Theorem \ref{theorem:optimality_constant_rate} and Theorem \ref{theorem:convergence}. Roughly speaking, smaller $\lrr$ allows us to take larger initialization, but we cannot take it arbitrarily small because it will cause numerical instability.
%
\subsection{Notation and outline}
\label{subsection:notation}
%
Boldface uppercase letters such as $\A$ are matrices with entries $\Ae_{\m\n}$, boldface lowercase letters such as $\w$ are vectors with entries $\we_{\n}$, and non-boldface letters are scalars. The pseudo-inverse of $\A$ is denoted by $\A^\dagger$. The transpose of $\A$ and $\w$ are denoted by $\A^\tT$ and $\w^\tT$. Orthogonal projection matrices are defined as $\PA:=\A^\dagger\A$ and $\Pw:=\frac{\w\w^\tT}{\|\w\|_2^2}$. Entry-wise products/powers are denoted with $\odot$, and inequalities between vectors and number are understood as entry-wise inequality, e.g. $\w\geq 0$ means $\we_\n\geq 0$ for all $\n$. Similarly, $\log(\w)$ is a vector whose entries are $\log(\we_\n)$, and the vector $\mathbf{1}$ and $\mathbf{0}$ are vectors of ones and zeros, respectively. We denote the set of non-negative real numbers as $\mathbb{R}_{+}$, and similarly the non-negative solution space as $S_{+}:=\{\z\geq 0: \A\z = \y\}$. We also denote the weighted $\ell_1$-norm of $\z$ with weight $\ww$ by $\|\z\|_{\ww,1} := \|\ww\odot\z\|_1$.

We state our main theorems in Section \ref{sec:Main_Result}. The main proofs are given in Section \ref{sec:Proof}. We demonstrate our numerical results in Section \ref{sec:Experiment}, and provide a summary and discussion in Section \ref{sec:Summary}.


\section{Main Results}
\label{sec:Main_Result}
Our main contribution is summarized in Theorem \ref{theorem:optimality_constant_rate}.  In general, the strength of implicit regularization increases as the (magnitude of) initialization decreases. Therefore many works focus on small \cite{chou2021more,stoger2021small} or infinitesimal \cite{arora2018optimization,arora2019implicit} initialization. However, small initialization causes not only numerical instability, but also slow convergence rate. Hence, there is a strong incentive to avoid small initialization.

The core idea is to magnify the implicit regularization via an appropriate learning rate ratio to obtain small error without small initialization $\g_0$. For instance, we will see that in the setting of Theorem \ref{theorem:optimality_constant_rate}, the error decreases exponentially with respect to initialization according to \eqref{eq:g0_constant_L1min_positive}, which can be compared to the setting of \cite{chou2021more} where the error only decreases polynomially according to \eqref{eq:alpha_bound_original}.

\begin{theorem}[Theorem 2.1 from \cite{chou2021more}]\label{theorem:L1_equivalence_positive}
    Let $\L\geq 2$, $\A\in\mathbb{R}^{\M\times\N}$ and $\y\in\mathbb{R}^{\M}$ and assume that $S_+$ is non-empty. Suppose $\x$ follows the dynamics \eqref{eq:dxdt} with $\x_0>0$. Let $\xprod=\x^{\odot L}$. Then the limit $\xprodinfty:= \lim_{\t\to\infty}\xprod(t)$ exists and $\xprodinfty \in S_+$. Moreover, let 
    \begin{align*}
        \ww = \xprod(0)^{\odot\frac{2}{\L}-1},\quad
        \Q := \min_{\z\in S_+}\|\z\|_{\ww,1},\quad
        \beta_{1} = \|\xprod(0)\|_{\ww,1},\quad
        \beta_{\min} = \min_{\n\in[\N]}\wwe_\n\xprode_\n(0).
    \end{align*}
    Suppose $\Q>\c_\L\beta_1^{\frac{2}{\L}}$, then
    \begin{equation}\label{eq:L1min_general_0}
        \|\xprodinfty\|_{\ww,1}-\Q\leq\epsilon\Q,
    \end{equation}
    where the constant $\c_\L$ is given by
    \begin{equation}\label{eq:cL}
        \c_\L := \begin{cases}
            1&\text{if }\L=2,\\
            \left(\frac{\L}{2}\right)^{\frac{\L}{\L-2}}&\text{if }\L>2,
        \end{cases}
    \end{equation}
    and the error $\epsilon$ is defined as
    \begin{align}\label{eq:eps}
        \epsilon(\beta_1,\beta_{\min})
        :=\begin{cases}
            \frac{\log(\beta_{1}/\beta_{\min})}{\log(\Q/\beta_{1})} &\text{if }\L=2,\\[6pt]
            \frac{\L(\beta_{1}^{1-\L/2} - \beta_{\min}^{1-\L/2})}{2\Q^{1-\L/2}- \L\beta_{1}^{1-\L/2}}&\text{if }\L>2.
        \end{cases}
    \end{align}
\end{theorem}
We now state our main result.
\begin{theorem}[Magnification of implicit regularization]\label{theorem:optimality_constant_rate}
Let $\L\in\mathbb{N}$, $\L\geq 2$, $\A\in\mathbb{R}^{\M\times\N}$, $\y\in\mathbb{R}^{\M}$, $(\eta_\g,\eta_\w) = (\lrr,1)$ for some constant $\lrr>0$. Suppose $(\g,\w)$ follow the dynamics in \eqref{eq:dgdt} and \eqref{eq:dwdt} with $\g_0,\w_0>0$ satisfying $\|\w_0\|_2 = 1$ and $\lrr^{1/2}\leq\g_0\leq\|\A^\dagger\b\|_{2}^{1/\L}$. Denote $\x=\frac{\g}{\|\w\|_2}\w$. Suppose there exists $\v>0$ such that $\A\v=0$, $S_{+}$ is non-empty, and the limit $\xprodinfty:= \lim_{\t\to\infty}\x(\t)^{\odot\L}$ exists. Let 
\begin{equation}
\label{eq:rho}
    \rho:=\frac{\g_0}{\|\A^\dagger\b\|_{2}^{1/\L}}\exp\left(\frac{\|\A^\dagger\b\|_{2}^{2/\L} - \g_0^2}{2\lrr}\right),
\end{equation}
be the magnification factor. Then
\begin{enumerate}
    \item The loss defined in \eqref{eq:factorized_loss} decreases exponentially in time, i.e., for all $\t\geq 0$
    \begin{equation}
        \Loss(\x(\t)) \leq \Loss(\x_0)e^{-\c\t}    
    \end{equation}
    for some constant $\c>0$. In addition, the limit $\xprodinfty$ lies is $S_+$.
    \item It holds that $\rho\geq 1$, and the limit $\xprodinfty$ satisfies \eqref{eq:L1min_general_0} with error
    \begin{equation}\label{eq:g0_constant_L1min_positive}
        \epsilon(\rho^{-\L}\beta_1,\rho^{-\L}\beta_{\min})
    \end{equation}
    as defined in \eqref{eq:eps}.
\end{enumerate}
\end{theorem}
%
%
\begin{proof}
See Section \ref{subsection:constant_rate}.
\end{proof}
Our second main result is that for $\L=2$ in particular, we can moreover prove convergence (rather than assuming it). 
\begin{theorem}[Convergence]\label{theorem:convergence}
    In the setting of Theorem \ref{theorem:optimality_constant_rate}, the limit $\xprodinfty:= \lim_{\t\to\infty}\xprod(\t)$ always exists for $\L=2$ (and there is no need to assume that it exists).
\end{theorem}
\begin{proof}
    See Section \ref{subsection:convergence for L=2}.
\end{proof}
Here are some remarks on the scaling and generalization of Theorem \ref{theorem:optimality_constant_rate} and Theorem \ref{theorem:convergence}.
\begin{itemize}
    \item \textbf{Error reduction due to $\rho$.}\\
    Consider the case where $\L=2$ and $\g_0=1$. Let $\Q=\min_{\z\in S_+} \|\z\|_{1}$. If the initialization takes the form $\alpha\bo$, then from Theorem \ref{theorem:optimality_constant_rate} the error is given by  
    \begin{equation*}
        \epsilon(\rho^{-\L}\beta_1,\rho^{-\L}\beta_{\min}) = \frac{\log(\beta_{1}/\beta_{\min})}{\log(\rho^{\L}\Q/\beta_{1})}.
    \end{equation*}
    Thus we can make the right hand side small by making $\rho$ sufficiently large. Note that this is impossible without normalization.
    \item \textbf{Enlarged range of initialization for sparse recovery.}\\
    Based on the error reduction, below we show that there is a wider range of initialization that yields similar (or even smaller error) for weight normalized GF than the regular GF. The ratio is given by $\|\A^\dagger\b\|_2/\delta$, where $\delta\ll1$ is the small radius for GD to achieve sufficient implicit regularization. From \eqref{eq:alpha_bound_original} and \eqref{eq:alpha_bound_wn}, we know that the equivalent level sets for errors are given by
    \begin{equation}\label{eq:delta_ball}
    \begin{cases}
        &\text{GD: }\{\xprod_0:\|\xprod_0\|_1\leq\delta\}\\
        &\text{WN-GD: }\{(\g_0\w_0)^{\odot\L}:\rho^{-\L}\|(\g_0\w_0)^{\odot\L}\|_1\leq\delta\}.
    \end{cases} 
    \end{equation}
    Consider $\L=2$. Since $\|\w_0\|_2 = 1$, if $\z = (\g_0\w_0)^{\odot 2}$, then $\|\z\|_1 = \|(\g_0\w_0)^{\odot 2}\|_1 = \g_0^2$. Hence
    \begin{align*}
        &\{(\g_0\w_0)^{\odot 2}:\rho^{-2}\|(\g_0\w_0)^{\odot 2}\|_1\leq \delta\}\\
        &\qquad= \left\{(\g_0\w_0)^{\odot 2}:\frac{\|\A^\dagger\b\|_{2}}{\g_0^2}\exp\left(\frac{\g_0^2-\|\A^\dagger\b\|_{2}}{\lrr}\right)\|(\g_0\w_0)^{\odot 2}\|_1\leq \delta\right\}\\
        &\qquad= \left\{\z:\frac{\|\A^\dagger\b\|_{2}}{\|\z\|_1}\exp\left(\frac{\|\z\|_1-\|\A^\dagger\b\|_{2}}{\lrr}\right)\|\z\|_1\leq \delta\right\}\\
        &\qquad= \left\{\z:\exp\left(\frac{\|\z\|_1-\|\A^\dagger\b\|_{2}}{\lrr}\right)\leq \frac{\delta}{\|\A^\dagger\b\|_{2}}\right\}\\
        &\qquad= \{\z:\|\z\|_1\leq \|\A^\dagger\b\|_{2}+\lrr(\log\delta-\log(\|\A^\dagger\b\|_{2}))\}
    \end{align*}
    Hence the radius of the $\ell_1$-ball increases from $\delta$ to $\|\A^\dagger\b\|_2+\lrr(\log\delta-\log(\|\A^\dagger\b\|_{2}))$, which leads to a huge improvement when $\delta\ll\|\A^\dagger\b\|_2$. Note that this is often the case because $\delta$ is usually chosen to be small to have small error $\epsilon$, whereas $\|\A^\dagger\b\|_2$ is not small in general.

    \item \textbf{Assumption on kernel of $\A$.}\\
    One of the assumptions is that the kernel of $\A \in \R^{\M \times N}$ has non-empty intersection with the positive quadrant.
    In general, the number of orthants in $\mathbb{R}^\N$ intersecting with a random subspace of dimension $K\geq 1$ is $2\cdot\sum_{i=0}^{K-1}\binom{\N-1}{i}$ \cite{Thomas1965geomtrical}. In our setting $\K$ is the dimension of the kernel of $\A$, and it is lower bounded by $\N-\M$. Suppose now that $\A$ is chosen at random such that the kernel is a random subspace, whose distribution is invariant under rotation. For instance, this is the case for Gaussian random matrices $\A$ as often considered in compressive sensing.
    Then by symmetry, the probability that it intersects the positive orthant is given by $2\sum_{i=0}^{K-1}\binom{\N-1}{i}/2^{\N}$. Thus our assumption holds with probability
    \begin{equation*}
        p = \frac{1}{2^{\N-1}}\cdot\sum_{i=0}^{K-1}\binom{\N-1}{i} = 1 - \frac{1}{2^{\N-1}}\cdot\sum_{i=0}^{\N-\K-1}\binom{\N-1}{i}.
    \end{equation*}
    For instance, when $\M=1$, $K=\N-1$ and hence $p=1 - (1/2)^{\N-1}$. Furthermore, when $\lambda:=\frac{\N-K-1}{\N-1}\leq \frac{1}{2}$ (small number of measurements) we can lower bound this probability via (\cite{Flum2006parameterized}) 
    \begin{align*}
        p \geq 1 - 2^{-(\N-1)(1+H(\lambda))},
    \end{align*}
    where $H(\lambda)= \lambda\log_2\lambda+(1-\lambda)\log_2(1-\lambda)$.
    \item \textbf{Extension beyond non-negative solutions}\\
    A common strategy to extend results from the positive solution set $S_+$ to the full solution set $S:=\{\z: \A\z = \y\}$ is to introduce further parameters. For example, although gradient flow on the loss function
    \begin{equation*}
        \mathcal{L}(\x) = \|\A\x^{\odot\L} - \y\|_2^2, \quad\x_0>0
    \end{equation*}
    can only lead to positive solutions, gradient flow on the modified loss function
    \begin{equation}\label{eq:loss pm}
        \mathcal{L}_\pm(\u,\v) = \|\A(\u^{\odot\L}- \v^{\odot\L}) - \y\|_2^2, \quad\u_0,\v_0>0
    \end{equation}
    can lead to any solution \cite{chou2021more,Gissin2019Implicit,woodworth2020kernel}, since $\u$ takes care of the positive part and $\v$ takes care of the negative part. However, our key lemma (Lemma \ref{lemma:unique}), which depends on uniqueness, no longer holds in this regime. The intuitive reason is that 
    \begin{equation*}
        \u^{\odot\L} - \v^{\odot\L}= (\u^{\odot\L}+\zeta) - (\v^{\odot\L}+\zeta)
    \end{equation*}
    for any $\zeta$. In particular, since we cannot assume $\u$ and $\v$ have disjoint supports, it is unlikely that we can uniquely define $(\u,\v)$ based on the invariants in Lemma \ref{lemma:inv_proj_constant_rate}. Therefore, our analysis only focuses on the case of positive solution set in this paper. We nevertheless include some simulation for gradient descent on the modified loss function \eqref{eq:loss pm} in Section \ref{subsection:Sparse ground truth with positive and negative entries} along with some discussion.
\end{itemize}
Theorem \ref{theorem:optimality_constant_rate} provides the optimal scaling when $\eta_\g$ is a constant. Interestingly, an alternative time-dependent step-size choice $(\eta_\g,\eta_\w) = (\g^2,1)$ provides a different dynamic which is close to the gradient flow dynamics without weight normalization \eqref{eq:dxdt}.  In this case, we can prove convergence to the stationary point instead of assuming it for all $L \geq 2$. Below we state the results for this special case and encourage readers to explore more possibilities.
\begin{theorem}[A time-dependent learning rate]\label{theorem:optimality_dynamic_rate}
Consider the same setting as in Theorem \ref{theorem:optimality_constant_rate}, except that the learning rates are given by $(\eta_\g,\eta_\w) = (\g^2,1)$ and here we do not assume the existence of the limit. If $\y$ is not identically zero, then the following holds.
\begin{enumerate}
    \item The limit $\xprodinfty:= \lim_{\t\to\infty}\xprod(\t)$ exists and lies in $S_{+}$.
    \item The limit $\xprodinfty$ satisfies \eqref{eq:L1min_general_0} with error $\epsilon(\beta_1,\beta_{\min})$ as defined in \eqref{eq:eps}. In other words, it satisfies the same error bound as in Theorem \ref{theorem:optimality_constant_rate} with $\rho = 1$ (no magnification).
\end{enumerate}
\end{theorem}
\begin{proof}
See Section \ref{subsection:dynamic_rate}. 
\end{proof}
%


\section{Proofs}
\label{sec:Proof}
In this section we will present the main lemma and theorems, along with some of the key proof techniques. The rest of the proofs can be found in Appendix \ref{sec:Appedix}. Here is the outline.
\begin{enumerate}
    \item Basic properties of the dynamics: we use Lemma \ref{lemma:rescaled_learning_rate} to reduce the number of parameters, Lemma \ref{lemma:constant_norm} to avoid division by zero, Lemma \ref{lemma:non_negative} to control the signs, Lemma \ref{lemma:loss_non_increasing} and Lemma \ref{lemma:Convergence rate} to guarantee that the loss is indeed decreasing.
    \item Proof of Theorem \ref{theorem:optimality_constant_rate}: the key proof technique is to compare the invariants with and without weight normalization (Lemma \ref{lemma:inv_proj_original}, \ref{lemma:inv_proj_constant_rate}, \ref{lemma:invariant_compare_1}). To ensure the well-posedness of such comparison, we prove the boundedness (Lemma \ref{lemma:Bounded above implies bounded below}) and the uniqueness (Lemma \ref{lemma:unique}) of the trajectories.
    \item Proof of Theorem \ref{theorem:convergence}: it relies on additional boundedness (Lemma \ref{lemma:boundedness}) for $\L=1,2$.
    \item Proof of Theorem \ref{theorem:optimality_dynamic_rate}: we use techniques similar to the ones in \cite{chou2021more}, where we analyze the trajectories through dynamic reduction (Lemma \ref{lemma:dynamic_reduction}) and the Bregman divergence (Lemma \ref{lemma:Bregman_divergence_non_increasing}.
\end{enumerate}
%
\subsection{Basic properties}
\label{subsection:basic_properties}
We first compute all the derivatives that will be used later on. By the chain rule,
\begin{align}
    \nabla\Loss(\x) &= [\A^\tT(\A\x^{\odot\L} - \y)] \odot (\x^{\odot\L-1}),\label{eq:dLdx}\\
    \nabla_{\g} \LossM(\g,\w) &= \frac{\w^\tT}{\|\w\|_2}\nabla\Loss\left(\frac{\g}{\|\w\|_2}\w\right),\label{eq:dLdg}\\
    \nabla_{\w} \LossM(\g,\w) &= \frac{\g}{\|\w\|_2}(\I-\Pw) \nabla\Loss\left(\frac{\g}{\|\w\|_2}\w\right).\label{eq:dLdw}
\end{align}
We claimed in Section \ref{subsection:steup} that $\eta_{\w}=1$ can be fixed without loss of generality. Lemma \ref{lemma:rescaled_learning_rate} provides the justification: scaling $\eta_{\w}$ is equivalent to scaling the magnitude of the  initialization $\w_0$, which is independent of the initialization $\x_0 = \frac{\g}{\|\w\|_2}\w$.
\begin{lemma}[Re-scaled learning rate $\eta_{\w}$]\label{lemma:rescaled_learning_rate}
Suppose $(\g,\w)$ follows \eqref{eq:dgdt} and \eqref{eq:dwdt} with initialization $(\g_0,\w_0)$ and learning rate $(\eta_\g,\eta_\w)$. Fix $\a>0$.
Suppose $(\g^{(\a)},\w^{(\a)})$ follows \eqref{eq:dgdt} and \eqref{eq:dwdt} with  initialization $(\g_0,\a\w_0)$ and learning rate $(\eta_\g,\a^2\eta_\w)$. Then $\g^{(\a)}(\t) = \g(\t)$ and $\w^{(\a)}(\t) = \a\w(\t)$ for all $\t\geq 0$. As a result,
\begin{align*}
    \frac{\g}{\|\w\|_2}\w = \frac{\g^{(\a)}\w^{(\a)}}{\|\w^{(\a)}\|_2}.
\end{align*}
\end{lemma}
\begin{proof}
    See Appendix \ref{sec:Appedix}.
\end{proof}
To avoid division by zero, it is important to have some control over the norm $\|\w\|_2$. According to the next lemma, $\|\w\|_2$ does not change in time, and hence it stays positive if it starts positive.
\begin{lemma}[Constant norm]\label{lemma:constant_norm}
For all $\t\geq 0$, $\|\w(\t)\|_2 = \|\w(0)\|_2$.
\end{lemma}
\begin{proof}
    See Appendix \ref{sec:Appedix}.
\end{proof}
Based on Lemma \ref{lemma:rescaled_learning_rate} and Lemma \ref{lemma:constant_norm}, without loss of generality we set $\|\w_0\|_2 = 1$ from now on so that $\|\w(\t)\|_2=1$ for all $\t\geq 0$.

Another important property is that $\w$ stays non-negative, as stated by the next lemma. This implies the property that $\langle \w, \1 \rangle = \|\w\|_1$, which will be useful in the proof of Theorem \ref{theorem:optimality_constant_rate} and Theorem \ref{theorem:optimality_dynamic_rate}.
\begin{lemma}[Constant sign]\label{lemma:non_negative}
Let $\L\in\mathbb{N}$ and $\L\geq 2$. If $\g(0)>0$ and $\w(0)>0$, then $\g(\t)>0$ and $\w(\t)>0$ for all $\t\geq 0$.
\end{lemma}
\begin{proof}
    See Appendix \ref{sec:Appedix}.
\end{proof}
\begin{remark}
    Note that although Lemma \ref{lemma:non_negative} shows that the $\w(\t)>0$ for all $\t\geq 0$, it does not guarantee that its limit is strictly positive, because hitting zeros will have some undesirable consequence. Hence later on while discussing the limit we will need some slightly stronger results (Lemma \ref{lemma:Bounded above implies bounded below} and the Lemma \ref{lemma:unique}). Note that although we will prove that the trajectories is bounded away from zero by some positive constant, this constant can be small when the initialization is small. Hence such bound does not prevent us from getting close to a sparse solution. The need of such constant is rather a technical requirement for the proof.
\end{remark}
We can show that the loss is non-increasing. Moreover, we can derive the convergence rate under the assumption that some entries are uniformly lower bounded, which in fact holds if the trajectory is upper bounded, as we will later on prove in Lemma \ref{lemma:Bounded above implies bounded below} in Section \ref{subsection:constant_rate}.
\begin{lemma}[Non-increasing loss]\label{lemma:loss_non_increasing}
If $(\g,\w)$ follows \eqref{eq:dgdt} and \eqref{eq:dwdt}, then the loss $\LossM(\g,\w)$ is non-increasing in time, i.e. $\partial_{\t}\LossM(\g,\w)\leq0$.
\end{lemma}
\begin{proof}
    See Appendix \ref{sec:Appedix}.
\end{proof}
\begin{lemma}[Convergence rate]\label{lemma:Convergence rate}
Let $\x=\frac{\g}{\|\w\|_2}\w$ and $(\g,\w)$ follows \eqref{eq:dgdt} and \eqref{eq:dwdt} with $\|\w_0\|_2^2=1$. Let $\t_0\geq0$. Suppose there exists constant $\c_\g,\c_\w,\c_\x>0$ such that $\eta_\g\geq\c_\g$, $\eta_\w\geq\c_\w$, and $|\x_I|\geq \c_\x$ for all $\t\in [\t_0,\T]$, for some index set $I\subset[\N]$. Denote $\A_I$ to be the sub-matrix of $\A$ with columns indexed by $I$. If $\A_I$ has full rank, then for all $\t\geq\t_0$,
\begin{equation}\label{eq:loss_non_increasing}
    \Loss(\x(\t)) \leq \Loss(\x(\t_0))\exp\left(-\min\left(\c_\g,\c_\w \c_\x^2|I|\right)2\L\c_\x^2\sigma_{\min}^2(\A|_I)(\t-\t_0)\right)   
\end{equation}
for all $\t\in[0,\T]$, where $\sigma_{\min}(\A|_I)$ is the smallest singular value of $\A|_I$.
\end{lemma}
\begin{proof}
    See Appendix \ref{sec:Appedix}.
\end{proof}
%
\subsection{Invariants of the flow}
\label{subsection:constant_rate}
%
We first derive some statements regarding invariants and uniqueness, and then present the proof of Theorem \ref{theorem:optimality_constant_rate}. The first step is to compare the invariants for the cases with and without normalization. We will see that they take similar form, but differ by an exponential factor.
\begin{lemma}[Invariant, without normalization]\label{lemma:inv_proj_original}
Let $\L\in\mathbb{N}$. Suppose $\x$ follows the dynamics in \eqref{eq:dxdt}.
Then the quantity
\begin{equation}\label{eq:inv_proj_original}
    \h_0 (\t):=
    \left(\I-\A^\dagger\A\right)\cdot\begin{cases}\log(\x(\t))&\text{if }\L=2\\
    \x(\t)^{\odot 2-\L}&\text{if }\L\neq 2
    \end{cases}
\end{equation}
is the same for all $\t\geq 0$, where $\A^\dagger$ is the pseudoinverse of $\A$.
\end{lemma}
\begin{proof}[Proof of Lemma \ref{lemma:inv_proj_original}]
Note that $(\I-\PA)\A^\tT = 0$. It suffices to show that $\partial_\t \h_0 = 0$. By direct computation we have
\begin{align*}
    \partial_\t \h_0
    &= (\I-\PA)\, c(\x^{\odot 1-\L}\odot \partial_\t\x )\\
    &= -c(\I-\PA)\, (\x^{\odot 1-\L}\odot [\A^\tT(\A\x^{\odot\L} - \y)] \odot \x^{\odot\L-1})\\
    &= -c(\I-\PA) \A^\tT(\A\x^{\odot\L} - \y)
    = 0
\end{align*}
where $c=1$ if $\L=2$, and $c=2-\L$ otherwise. Thus $\h_0$ remains constant for all $\t\geq 0$.
\end{proof}
\begin{lemma}[Invariant, with normalization]\label{lemma:inv_proj_constant_rate}
Let $\L\in\mathbb{N}$, $\lrr>0$, and $(\eta_\g,\eta_\w) = (\lrr,1)$. Suppose $\g,\w$ follow the dynamics in \eqref{eq:dgdt} and \eqref{eq:dwdt} with $\|\w_0\|_2 = 1$. Then the quantity
\begin{equation}\label{eq:inv_proj_constant_rate}
    \h_\lrr(\t)
    :=(\I-\A^\dagger\A)\cdot\begin{cases}
    \log\left(\w(\t)\exp\left(\frac{1}{2\lrr}\g(\t)^2\right)\right)&\text{if }\L=2\\
    \w(\t)^{\odot 2-\L}\exp\left(\frac{2-\L}{2\lrr}\g(\t)^2\right)&\text{if }\L\neq 2
    \end{cases}
\end{equation}
remains constant for all $\t\geq 0$, where $\A^\dagger$ is the pseudoinverse of $\A$.
\end{lemma}
\begin{proof}[Proof of Lemma \ref{lemma:inv_proj_constant_rate}]
Since $\|\w_0\|_2=1$, by Lemma \ref{lemma:constant_norm} $\|\w(\t)\|_2=1$ for all $\t\geq0$. By substituting $\eta_\g=\lrr$, $\eta_\w=1$, and $\|\w\|_2=1$ into the dynamics \eqref{eq:dw_dg_1}, we obtain
\begin{align}
    \partial_{\t}\w
    = -\g\nabla\Loss\left(\g\w\right) - \frac{\g}{\lrr}\w\partial_{\t}\g
    = -\g\nabla\Loss\left(\g\w\right) - \frac{1}{2\lrr}\partial_{\t}(\g^2)\w.\label{eq:dw_dg_2}
\end{align}
Since
\begin{equation}\label{eq:perp_kernel_A}
    (\I-\PA)\left(\w^{\odot 1-\L} \odot \nabla\Loss\left(\g\w\right)\right)
    = (\I-\PA)\left(\A^\tT\left(\A\left(\g\w\right)^{\odot\L} - \y\right)\g^{\L-1}\right) = 0,
\end{equation}
by applying the operation $(\I-\PA)[\w^{\odot 1-\L}\odot\,\cdot\,]$ to \eqref{eq:dw_dg_2} we have
\begin{align}\label{eq:proj_eq1}
    (\I-\PA)(\w^{\odot 1-\L} \odot \partial_{\t}\w)
    = -\frac{1}{2\lrr}\partial_{\t}(\g^2)(\I-\PA)\w^{\odot 2-\L}.
\end{align}
Let us now separate the case of $\L=2$ from $\L>2$.

For $\L=2$, since $\w^{\odot -1} \odot \partial_{\t}\w = \partial_{\t}\log(\w)$, \eqref{eq:proj_eq1} can be expressed as 
\begin{equation*}
    (\I-\PA)\partial_{\t}\log(\w)
    = -\frac{1}{2\lrr}\partial_{\t}(\g^2)(\I-\PA)\1,
\end{equation*}
which is a separable differential equation whose solution (via integration from $0$ to $\t$) satisfies
\begin{equation*}
    (\I-\PA)(\log(\w) - \log(\w_0)) = -\frac{1}{2\lrr}(\g^2- \g_0^2)\cdot(\I-\PA)\1.
\end{equation*}
Rearranging terms we obtain
\begin{equation}
    (\I-\PA)\left(\log(\w) + \frac{\g^2}{2\lrr}\1\right) = (\I-\PA)\left(\log(\w_0) + \frac{\g_0^2}{2\lrr}\1\right).
\end{equation}
Note that we can combine terms by noting that
\begin{align*}
    \log(\w) + \frac{\g^2}{2\lrr}\1
    = \log(\w) + \log\left(\exp\left(\frac{\g^2}{2\lrr}\1\right)\right)
    = \log\left(\w \cdot \exp\left(\frac{\g^2}{2\lrr}\right)\right)
\end{align*}
For $\L\neq 2$ the left hand side of \eqref{eq:proj_eq1} can be written as $\frac{1}{2-\L}\partial_{\t}(\I-\PA)\w^{\odot 2-\L}$. Let $\wprod = (\I-\PA)\w^{\odot 2-\L}$. Then we have that
\begin{align*}
    \partial_{\t} (\wprod) = -\frac{2-\L}{2\lrr}\partial_{\t}(\g^2)\cdot\wprod.
\end{align*}
Multiplying both sides by $\odot\wprod^{-1}$ and again using the fact that $\wprod^{\odot -1} \odot \partial_{\t}\wprod = \partial_{\t}\log(\wprod)$, we have
\begin{equation*}
    \partial_{\t} (\log (\wprod))
    = \frac{\L-2}{2\lrr}\partial_{\t}(\g^2)\cdot\1,
\end{equation*}
which is a separable differential equation whose solution (via integration from $0$ to $\t$) satisfies
\begin{equation*}
    \log(\wprod) - \log(\wprod_0)
    =\frac{\L-2}{2\lrr}(\g^2- \g_0^2)\cdot\bo.
\end{equation*}
Since $\log(\wprod) - \log(\wprod_0)=\log(\wprod\odot\wprod_0^{\odot -1})$, we take the exponential on both sides to get
\begin{equation*}
    \wprod\odot\wprod_0^{\odot -1}
    =\exp\left(\frac{\L-2}{2\lrr}(\g^2- \g_0^2)\right)\cdot\bo,
\end{equation*}
which is equivalent to
\begin{align}
    \wprod\cdot\exp\left(\frac{2-\L}{2\lrr}\g^2\right)
    =\wprod_0\cdot\exp\left(\frac{2-\L}{2\lrr}\g_0^2\right)
\end{align}
and hence the conclusion follows.
\end{proof}
Lemma \ref{lemma:inv_proj_constant_rate} is a generalization of Lemma 2.5 in \cite{Wu2019implicit}, which considered the case $\L=1$ corresponding to linear regression. When $\L=1$, this means that the component  $\left(\I-\A^\dagger\A\right)\w$ vanishes as $\g$ increases, so that $\g_\infty\w_\infty \approx \A^\dagger\y$. For $\L\neq 1$, the geometric interpretation is less intuitive because we only have the characterization of $\w^{\odot 2-\L}$ instead of $\w$.

Fortunately, instead of directly analyzing the invariant $\h_\lrr$ in Lemma \ref{lemma:inv_proj_constant_rate}, we can compare it with the invariants $\h_0$ in Lemma \ref{lemma:inv_proj_original} and make an insightful connection, which will be stated in the next lemma.
\begin{lemma}[Invariant comparison]\label{lemma:invariant_compare_1}
    Let $\L\in\mathbb{N}$ and $\lrr>0$. Suppose $\x$ follows the dynamics in \eqref{eq:dxdt}, and $(\g,\w)$ follow the dynamics in \eqref{eq:dgdt} and \eqref{eq:dwdt} with $(\eta_\g,\eta_\w) = (\lrr,1)$, $\g_0,\w_0>0$, and $\|\w_0\|_2 = 1$. Denote $\xwn = \frac{\g}{\|\w\|_2}\w$. Suppose $\g_\infty = \lim_{\t\to\infty}\g(\t)$ exists and not equal to zero. Then for $\L=2$,
    \begin{align}
        \lim_{\t\to\infty}\left(\I-\PA\right)\log(\x(\t)) &= \left(\I-\PA\right)\log(\x(0))\\
        \lim_{\t\to\infty}\left(\I-\PA\right)\log(\xwn(\t)) &= (\I-\PA)\log\left(\r(\g_0,\ginfty)\cdot \xwn(0)\right)
    \end{align}
    and for $\L\neq 2$,
    \begin{align}
        \lim_{\t\to\infty}\left(\I-\PA\right)\x^{\odot2-\L}(\t) &=\left(\I-\PA\right)\x^{\odot 2-\L}(0)\\
        \lim_{\t\to\infty}\left(\I-\PA\right)\xwn^{\odot 2-\L}(\t) &=\left(\I-\PA\right)[\,\r(\g_0,\ginfty)\cdot\xwn^{\odot\L}(0)]^{\odot 2-\L}
    \end{align}
    where the re-scaling factor is given by
    \begin{align}\label{eq:re-scaling factor}
        \r(\g_0,\g) &:= \frac{\g}{\g_0}\exp\left(\frac{\g_0^2-\g^2}{2\lrr}\right).
    \end{align}
\end{lemma}
We only focus on $\lim_{\t\to\infty}\left(\I-\PA\right)\log(\x(\t))$ and not $\left(\I-\PA\right)\log(\lim_{\t\to\infty}\xprod(\t))$, because the latter quantity might not be well-defined if $\lim_{\t\to\infty}\xprod(\t)$ has zero entries.
\begin{proof}[Proof of Lemma \ref{lemma:invariant_compare_1}]
    For $\x$, the result directly follows from Lemma \ref{lemma:inv_proj_original}. For $\xwn$, we need to do a bit more calculation. By Lemma \ref{lemma:constant_norm}, $\|\w(\t)\|_2=1$ for all $\t\geq 0$. Thus $\xwn=\g\w$.
    By assumption the limit $\ginfty$ exists and is strictly positive. By Lemma \ref{lemma:inv_proj_constant_rate}, for $\L=2$, 
    \begin{align*}
        \left(\I-\PA\right)\log(\xwn)
        &= (\I-\PA)\log(\g\w)\\
        &= (\I-\PA)\log\left(\w\exp\left(\frac{\g^2}{2\lrr}\right)\exp\left(-\frac{\g^2}{2\lrr}\right)\g\right)\\
        &= \underbrace{(\I-\PA)\log\left(\w \exp\left(\frac{\g^2}{2\lrr}\right)\right)}_{\text{invariant}} +(\I-\PA)\1\log\left(\exp\left(-\frac{\g^2}{2\lrr}\right)\g\right)\\
        &= (\I-\PA)\log\left(\w_0 \exp\left(\frac{\g_0^2}{2\lrr}\right)\right) +(\I-\PA)\1\log\left(\exp\left(-\frac{\g^2}{2\lrr}\right)\g\right)\\
        &= (\I-\PA)\log\bigg(\underbrace{\g_0\w_0}_{\xwn(0)}\underbrace{\frac{\g}{\g_0}\exp\left(\frac{\g_0^2-\g^2}{2\lrr}\right)}_{\r(\g_0,\g)} \bigg)\\
        &= (\I-\PA)\log\left(\r(\g_0,\g)\cdot \xwn(0)\right)
    \end{align*}
    and for $\L\neq 2$ we obtain that
    \begin{align*}
        \left(\I-\PA\right)\xwn^{\odot 2-\L}
        &= \left(\I-\PA\right)\w^{\odot 2-\L}\g^{ 2-\L}\\
        &= \underbrace{\left(\I-\PA\right)\w^{\odot 2-\L}\exp\left(\frac{2-\L}{2\lrr}\g^2\right)}_{\text{invariant}}\exp\left(-\frac{2-\L}{2\lrr}\g^2\right) \g^{ 2-\L}\\
        &= \left(\I-\PA\right)\w_0^{\odot 2-\L}\exp\left(\frac{2-\L}{2\lrr}\g_0^2\right)\exp\left(-\frac{2-\L}{2\lrr}\g^2\right) \g^{ 2-\L}\\
        &= \left(\I-\PA\right)(\,\underbrace{\w_0\g_0}_{\xwn(0)}\,)^{\odot 2-\L} \underbrace{\left(\frac{\g}{\g_0}\right)^{2-\L}\exp\left(\frac{(2-\L)(\g_0^2-\g^2)}{2\lrr}\right)}_{\r(\g_0,\g)^{2-\L}}\\
        &= \left(\I-\PA\right)[\,\r(\g_0,\g)\cdot\xwn(0)]^{\odot 2-\L}.
    \end{align*}
    By continuity $\r(\g_0,\g)$ converges to $\r(\g_0,\ginfty)$, which is well-defined because $\g_0,\ginfty>0$. This completes the proof.
\end{proof}
To make the full use of Lemma \ref{lemma:invariant_compare_1}, we need to ensure that the invariants and the fact that they converge to zero loss uniquely characterize the relation between $\xwn$ and $\x$. Thus we will need the following two lemmas.
Note that for the first one, we need to assume that the rows of $\A$ sum to zero. We leave it to future investigations whether boundedness from below of the entries of $\xwn$ and $\x$ holds also in general, or under other conditions.
\begin{lemma}[Bounded above implies bounded below]\label{lemma:Bounded above implies bounded below}
    Consider the same setting as in Lemma \ref{lemma:invariant_compare_1}. Suppose that there exists $\v > 0$ such $\A \v =0$ and that $\xwn,\x$ are bounded above. Then each entry of $\xwn$ and $\x$ is also bounded away from zero.
\end{lemma}
\begin{proof}
    See Appendix \ref{sec:Appedix}.
\end{proof}
\begin{lemma}[Uniqueness]\label{lemma:unique}
    Suppose $\xprod^{(1)},\xprod^{(2)}$ are strictly positive and uniformly bounded above and away from zero. If
    \begin{align}\label{eq:agree_PA}
        \lim_{\t\to\infty}\A[\xprod^{(1)}(\t)-\xprod^{(2)}(\t)]=0
    \end{align}
    and
    \begin{align}\label{eq:agree_I-PA}
    \begin{cases}
        \lim_{\t\to\infty}\left(\I-\PA\right)[\log(\xprod^{(1)}(\t))-\log(\xprod^{(2)}(\t))]=0
        &\text{if }\L=2,\\
        \lim_{\t\to\infty}\left(\I-\PA\right)[(\xprod^{(1)}(\t))^{\odot \frac{2}{\L}-1}-(\xprod^{(2)}(\t))^{\odot \frac{2}{\L}-1}]=0
        &\text{if }\L\neq 2,
    \end{cases}
    \end{align}
    then $\lim_{\t\to\infty}\xprod^{(1)}(\t)-\xprod^{(2)}(\t)=0$.
\end{lemma}
\begin{proof}
    See Appendix \ref{sec:Appedix}.
\end{proof}
\begin{proof}[Proof of Theorem \ref{theorem:optimality_constant_rate}]
    The convergence rate of loss directly follows from Lemma \ref{lemma:Convergence rate} and Lemma \ref{lemma:Bounded above implies bounded below}. For notation simplicity, denote $\xprod = \x^\L$ and $\xwnprod = \left(\g\w/\|\w\|_2\right)^{\odot\L}$. Since we assume that $\xwnprod$ converges to a minimizer of the loss function, we obtain
    \begin{equation*}
        0 = \lim_{\t\to\infty}\PA[\xprod(\t) - \xwnprod(\t)].
    \end{equation*}
    By Lemma \ref{lemma:invariant_compare_1}, we also have 
    \begin{align*}
        0 &= \lim_{\t\to\infty}\left(\I-\PA\right)
        \begin{cases}
        [\log(\xprod(\t))-\log(\xwnprod(\t))]
        &\text{if }\L=2,\\
        [\xprod^{\odot \frac{2}{\L}-1}(\t) - \xwnprod^{\odot \frac{2}{\L}-1}(\t)]
        &\text{if }\L\neq 2.
        \end{cases} 
    \end{align*}
    Therefore we can show that $\lim_{\t\to\infty}\xwnprod(\t) = \lim_{\t\to\infty}\xprod(\t)$. We use Theorem 2.1 from \cite{chou2021more}, which characterize the limit of $\xprod$, to draw the conclusion on $\xwnprod$. Note that $\xprod$ is uniformly bounded below according to Lemma \ref{lemma:Bounded above implies bounded below}.
    
    Essentially, we effectively re-scale the initialization by $\r(\g_0,\ginfty)$. To obtain the rest of the theorem, we will analyze the function $\r$. The goal is to minimize $\r(\g_0,\ginfty)$ so that the ``effective'' initialization is small, ideally much less than $1$, so that we get a weaker bound than the one in Theorem \ref{theorem:L1_equivalence_positive} by a factor of $\r(\g_0,\ginfty)$. Since $\g_0$ and $\ginfty$ are dependent but we do not know the exact relation, we will use some properties of the $\r$ and $\ginfty$ to derive bound of the improved factor.

    We now examine the relation between $\g_0$, $\ginfty$, and $\|\A^\dagger\b\|_{2}^{1/\L}$. Because $\A^\dagger\b\in\argmin_{\A\z=\b}\|\z\|_2$ and  $\A\xprodinfty=\b$, we have
    \begin{equation}\label{eq:g0_ginfty_Adaggerb}
        \ginfty
        = \|\xprodinfty^{\odot 1/\L}\|_2 = \|\xprodinfty\|_{2/\L}^{1/\L}
        \geq \|\xprodinfty\|_{2}^{1/\L}
        \geq \|\A^\dagger\b\|_{2}^{1/\L}\geq\g_0.
    \end{equation}
    Recall that
    \begin{equation*}
        \r(\g_0,\g) = \frac{\g}{\g_0}\exp\left(\frac{\g_0^2-\g^2}{2\lrr}\right).
    \end{equation*}
    The partial derivative of $\r$ is given by 
    \begin{align*}
        \partial_{\g}\r(\g_0,\g)
        &= \frac{1}{\g_0}\left(1 - \frac{\g^2}{\lrr}\right)\cdot\exp\left(\frac{\g_0^2-\g^2}{2\lrr}\right).
    \end{align*}
    Note that $\partial_{\g}\r(\g_0,\g)\leq 0$ for $\g\geq\sqrt{\lrr}$. By \eqref{eq:g0_ginfty_Adaggerb}, we have $\ginfty\geq\|\A^\dagger\b\|_{2}^{1/\L}\geq\g_0\geq\sqrt{\lrr}$ and hence
    \begin{equation*}
        \r(\g_0,\ginfty)
        \leq \r(\g_0,\|\A^\dagger\b\|_{2}^{1/\L})
        \leq \r(\g_0,\g_0) = 1.
    \end{equation*}
    Thus $\r(\g_0,\|\A^\dagger\b\|_{2}^{1/\L})$ is an upper bound of $\r(\g_0,\ginfty)$. Because $\rho = \r^{-1}$ represents the improvement under weight normalization (larger $\rho$ is better), $\r(\g_0,\|\A^\dagger\b\|_{2}^{1/\L})^{-1}$ is a lower bound, or a minimal guarantee, for the improvement we will get. This completes the proof.
\end{proof}
%
\subsection{Convergence for L=1,2}
\label{subsection:convergence for L=2}
%
In this section we will prove the boundedness for $\L=1,2$ state Lojasiewicz’s Theorem \cite{convergence2005}, and based on this we will prove the convergence result stated in Theorem \ref{theorem:convergence}.
\begin{lemma}[Boundedness]\label{lemma:boundedness}
    Let $\L=1,2$ and $\lrr>0$. Suppose $(\g,\w)$ follow the dynamics in \eqref{eq:dgdt} and \eqref{eq:dwdt} with $(\eta_\g,\eta_\w) = (\lrr,1)$, $\g_0,\w_0>0$. If all entries of $\u$ are bounded away from zero, then $\g$ is uniformly upper bounded.
\end{lemma}
\begin{proof}
    See Appendix \ref{sec:Appedix}.
\end{proof}
\begin{theorem}[Theorem 4 in \cite{BahLearning2021}]\label{theorem:Lojasiewicz}
    If $\mathcal{L}:\mathbb{R}^\N\to\mathbb{R}$ is analytic and the curve $\t\mapsto\x(\t)\in\mathbb{R}^\N$, $\t\in[0,\infty)$ is bounded and a solution of the gradient flow equation $\partial_\t\x= -\nabla \mathcal{L}(\x)$, then $\x$ converges to a critical point of $\mathcal{L}$ as $\t\to\infty$.
\end{theorem}
\begin{proof}[Proof of Theorem \ref{theorem:convergence}]
    The assumptions of Theorem \ref{theorem:Lojasiewicz} are satisfied with loss function $\LossM$. By Lemma \ref{lemma:boundedness}, $\|\x(\t)\|$ is bounded, and hence by Theorem \ref{theorem:Lojasiewicz} must converges to a critical point of $\LossM$.
\end{proof}
%
\subsection{An example of time-dependent learning rate}
\label{subsection:dynamic_rate}
%
In this section we study a particular example of time-dependent learning rate, given by $(\eta_\g,\eta_\w) = (\g^2,1)$. Note that instead of a constant in time, $\eta_\g=\g^2(\t)$ is a function depends on time. In this case the dynamics is greatly simplified and is similar to gradient flow without normalization \eqref{eq:dxdt}. Such simplification allows us to analyze the dynamics based on established methods, such as the argument with Bregman divergence in \cite{chou2021more}, and completely bypass the need of invariants and uniqueness results proved in Section \ref{subsection:constant_rate}.

However, gradient flow under this particular choice of learning rate ($(\eta_\g,\eta_\w) = (\g^2,1)$) no longer exhibits the magnification effect as in the constant rate case ($(\eta_\g,\eta_\w) = (\lrr,1)$), and hence does not yield better bounds than the ones in previous works \eqref{eq:alpha_bound_original}. It is nevertheless remarkable that the dynamics with certain choices of learning rate can be so different from the one with time-independent learning rate.

We will first prove a general reduction in dynamics. In fact, from Lemma \ref{lemma:dynamic_reduction} we can see why $(\eta_\g,\eta_\w) = (\g^2,1)$ is a natural choice of time-dependent learning rate.
\begin{lemma}[Dynamics reduction]\label{lemma:dynamic_reduction}
Suppose $\g$ and $\w$ follow the gradient flow in \eqref{eq:dgdt} and \eqref{eq:dwdt} with $\|\w(0)\|=1$. Let $\x:=\frac{\g}{\|\w\|_2}\w$. Then
\begin{equation}\label{eq:dynamic_reduction}
    \partial_{\t}\x
    = -\left(\eta_\g\frac{\x\x^\tT}{\|\x\|_2^2} + \eta_\w(\|\x\|_2^2\I-\x\x^\tT)\right)\nabla\Loss\left(\x\right).
\end{equation}
\end{lemma}
\begin{proof}
    See Appendix \ref{sec:Appedix}.
\end{proof}
Observe that if $(\eta_\g,\eta_\w) = (\g^2,1)$, then \eqref{eq:dynamic_reduction} is greatly simplified as stated next.
\begin{lemma}[Dynamic reduction, time-dependent learning rate]\label{lemma:dynamic_reduction_dynamic_rate}
Let $\L\in\mathbb{N}$, $(\eta_\g,\eta_\w) = (\g^2,1)$, and $\|\w(0)\|=1$. Denote $\x:=\frac{\g}{\|\w\|_2}\w$. Then
\begin{equation}\label{eq:dynamic_reduction_dynamic_rate}
    \partial_{\t}\x
    = -\|\x\|_2^2\cdot\nabla\Loss\left(\x\right).
\end{equation}
\end{lemma}
\begin{proof}
Apply Lemma \ref{lemma:dynamic_reduction} with $(\eta_\g,\eta_\w) = (\g^2,1)$ in \eqref{eq:dynamic_reduction}.
\end{proof}
Note that without the additional factor $\|\x\|_2^2$, we are back to the setting that does not include weight normalization at all \eqref{eq:overloss_2}. Although the extra factor changes in time, since it is only a scalar as opposed to a vector or a matrix, it is possible to apply the same proof strategy as for gradient flow without normalization as in \cite{chou2021more}.

Before diving into the proof we would like to outline the general concepts. We consider the set of all non-negative solutions $S_{+}:=\{\z\geq0:\A\z = \y\}$ and examine how $\xprod:=\x^{\odot\L}$ approaches this set. The key insight is to measure the distance with an appropriate Bregman divergence, so that $\xprod$ approaches every element in $S_{+}$ {\bf at the same rate}. Hence, by proving that $\xprod$ eventually reaches $S_{+}$, we conclude that it must reach the element that is closest to the initialization measured in Bregman divergence defined next. 
\begin{definition}[Bregman Divergence]\label{def:Bregman_Divergence}
Let $\F:\Omega\to\mathbb{R}$ be a continuously-differentiable, strictly convex function defined on a closed convex set $\Omega$. The Bregman divergence associated with $\F$ for points $p,q\in\Omega$ is defined as
\begin{equation}\label{eq:Bregman_divergence_general}
    D_{\F}(p,q) = \F(p) - \F(q) - \langle \nabla \F(q), p-q \rangle.
\end{equation}
\end{definition}
\begin{lemma}[\cite{Bregman1967}]\label{lemma:Bregman_divergence}
  The Bregman divergence $D_\F$ is non-negative and, for any $q \in \Omega$, the function $p \mapsto D_F(p,q)$ is strictly convex.
\end{lemma}
\begin{lemma}[Non-increasing Bregman Divergence]\label{lemma:Bregman_divergence_non_increasing}
Let $\L\in\mathbb{N}$, $\L\geq 2$, $(\eta_\g,\eta_\w) = (\g^2,1)$, and $\|\w(0)\|=1$. Denote $\x=\frac{\g}{\|\w\|_2}\w$ and $\xprod = \x^{\odot\L}$. Then for any $\z\geq 0$ such that $\A\z=\y$,
\begin{equation}\label{eq:Bregman_divergence_non_increasing}
    \partial_{\t} D_{\F}(\z,\xprod) = -2\L\|\x\|_2^2\cdot\Loss(\x)
\end{equation}
where $D_{\F}$ is the Bregman divergence associated with the function $\F:\mathbb{R}_{+}^{\N}\to\mathbb{R}$ given by
\begin{equation}\label{eq:Bregman_divergence_poly}
    \F(\xprod) = \begin{cases}
    \frac{1}{2} \langle \xprod \odot \log(\xprod) - \xprod, \1 \rangle & \text{if }\L = 2\\
    \frac{\L}{2(2-\L)} \langle \xprod^{\odot \frac{2}{\L}}, \1 \rangle & \text{if }\L > 2.
    \end{cases}
\end{equation}
\end{lemma}
\begin{proof}
    This directly comes from Lemma \ref{lemma:dynamic_reduction_dynamic_rate} and Definition \ref{def:Bregman_Divergence}.
\end{proof}
\begin{lemma}[Convergence, time-dependent learning rate]\label{lemma:convergence_dynamic_rate}
Let $\L\in\mathbb{N}$, $\L\geq 2$, $(\eta_\g,\eta_\w) = (\g^2,1)$, $\|\w(0)\|=1$ and $\g(0)>0$. Denote $\x=\frac{\g}{\|\w\|_2}\w$. Suppose that $\y$ is not identically zero and the solution set $S_{+} = \{\z\geq0:\A\z=\y\}$ is non-empty. Then $\lim_{\t\to\infty} \Loss(\x(\t)) = 0$.
\end{lemma}
\begin{proof}
    See Appendix \ref{sec:Appedix}.
\end{proof}
\begin{proof}[Proof of Theorem \ref{theorem:optimality_dynamic_rate}]
The convergence of the loss follows directly from Lemma \ref{lemma:convergence_dynamic_rate}. Then existence of the limit and the optimality of the limit follows the same proof strategy as in \cite{chou2021more}. 
\end{proof}
%


\section{Experiments}
\label{sec:Experiment}

In this section we test our method across different number of layers $\L$, learning rate ratio $\lrr$, and initialization scale
\begin{equation}\label{eq:alpha}
    \alpha:=\|(\g_0\w_0)^{\odot\L}\|_1=\|\xprod_0\|_1.
\end{equation}
We will focus two things: the comparison of the reconstruction error between GD and WN-GD (GD with weight normalization), and how the learning rate ratio $\lrr$ affects the reconstruction error.

We set the ambient dimension to be $\N=1000$ and $\M=150$. The matrix $\A$ is generated as
\begin{equation*}
    \A = \frac{1}{\sqrt{\M}}{\bf G},\quad \b = \A\x^*
\end{equation*}
where $\x^*$ is the ground truth and ${\bf G}\in\mathbb{R}^{\M\times\N}$ has independent and standard normal distributed entries. Fix $\s=10$. We examine the case where $\x^*\geq 0$ is $\s$-sparse and has $\ell_1$-norm equals to $\s$. All experiments are conducted with constant small step size. Each data point is an average over ten instances of random data and random initialization $\x_0$.

The reconstruction is defined as
\begin{equation*}
    \epsilon_1:=\|\xprodinfty - \x^*\|_1.
\end{equation*}
By Theorem \ref{theorem:optimality_constant_rate}, the difference
$\|\xprodinfty\|_1-\min_{\z\in S_{+}} \|\z\|_1$ should be small for small initialization. Due to the restricted isometry property of $\A$, the reconstruction error should also decrease as the initialization decreases.

In the first experiment we compare GD and WN-GD with fixed $\L=2$ and $\lrr = 0.1$. In the second experiment we record the reconstruction error for different $\lrr$ with fixed number of layer $\L=2$ and initialization $\alpha=1$. In the third experiment we examine the performance of WN-GD for $\L=2$ and $\L=3$ with fixed $\lrr = 0.1$. In the forth experiment we no longer require $\x^*$ to be non-negative, and perform GD and WN-GD according to the loss function \eqref{eq:loss pm}.

To keep the presentation concise, we will not plot the results for weight normalization with time-dependent learning rate specified in Section \ref{subsection:dynamic_rate}, because in all our experiments they perfectly coincide with the results coming from regular GD without weight normalization, which is consistent with Theorem \ref{theorem:optimality_dynamic_rate}.

%
\subsection{Compare GD and WN-GD}
\label{subsection:Gaussian ground truth}
%
\begin{figure}[h]
    \centering
    \begin{subfigure}[b]{0.47\textwidth}
        \centering
        \includegraphics[width = \textwidth]{GD_WNGD_TestError_Initialization.png}
        \subcaption{WN-GD achieves smaller reconstruction error.}
        \label{fig:WN_TestError_Initialization}
    \end{subfigure}
    \hfill
    \begin{subfigure}[b]{0.485\textwidth}
        \centering
        \includegraphics[width = \textwidth]{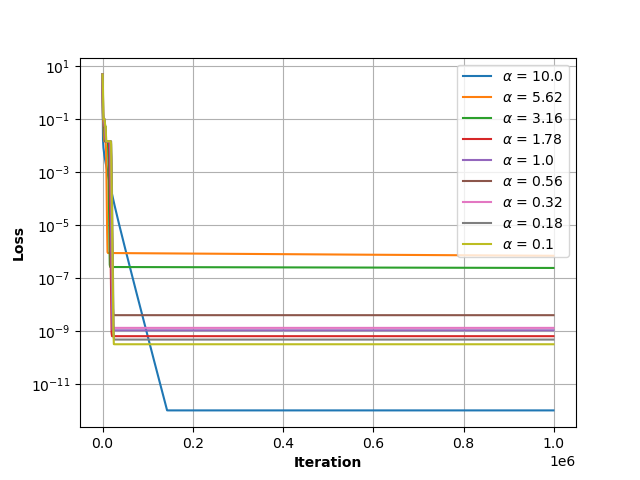}
        \subcaption{Training loss of WN-GD for different initialization.}
        \label{fig:WN_TrainError}
    \end{subfigure}
    \caption{WN-GD yields significantly much smaller error than GD. The training loss converges to values close to zero.}
    \label{fig:exp1}
\end{figure}

The goal of the first experiment (Figure \ref{fig:exp1}) is to compare GD and WN-GD among different initialization. We fix $\L=2$ and $\lrr = 0.1$. In Figure \ref{fig:WN_TestError_Initialization}, we compare the error $\epsilon_1$ produced by the two algorithms. According to Theorem \ref{theorem:optimality_constant_rate}, for any fixed initialization satisfying the conditions in Theorem \ref{theorem:optimality_constant_rate}, WN-GD should yield smaller $\epsilon_1$. In fact, the difference is quite significant.

In Figure \ref{fig:WN_TrainError} we observe that during training all loss decreases monotonically to vicinity of zero. According to experiments, convergence usually holds as long as the step size is sufficiently small and the initialization is smaller than the norm of the limit. 
%
\subsection{Compare learning rate ratio}
\label{subsection:learning rate ratio}
%
\begin{figure}[h]
    \centering
    \begin{subfigure}[b]{0.47\textwidth}
        \centering
        \includegraphics[width = \textwidth]{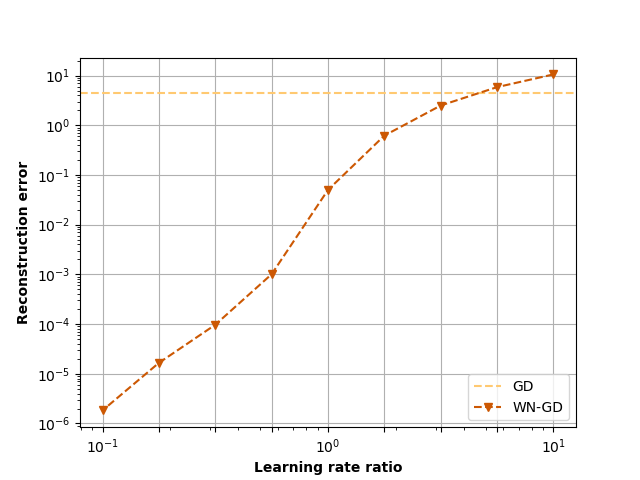}
        \subcaption{Smaller $\lrr$ leaders to smaller error (for WN-GD).}
        \label{fig:WN_TestError_LearningRateRatio_Grid_2}
    \end{subfigure}
    \hfill
    \begin{subfigure}[b]{0.47\textwidth}
        \centering
        \includegraphics[width = \textwidth]{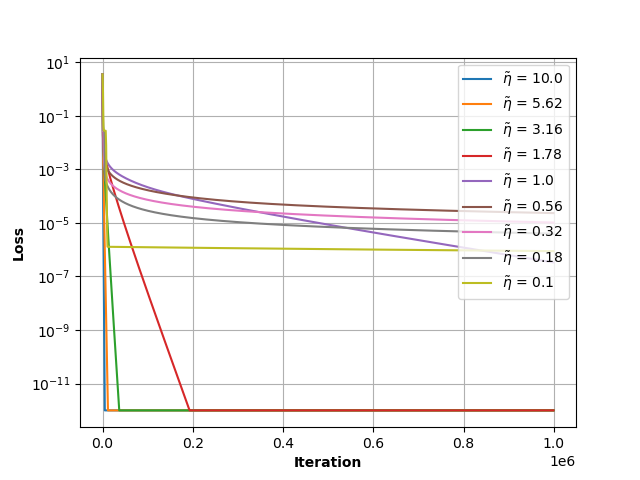}
        \subcaption{Training loss of WN-GD for different learning rate ratio.}
        \label{fig:WN_TrainError_LearningRateRatio}
    \end{subfigure}
    \caption{As the learning rate ratio $\lrr$ decreases, the error decreases. Note that when $\lrr=10$, the assumption of Theorem \ref{theorem:optimality_constant_rate} is violated, and we see that WN-GD is not better than GD.
    }
    \label{fig:exp2}
\end{figure}
In the second experiment (Figure \ref{fig:exp2}) we aim to understand how $\lrr$ affects the reconstruction error. Here we fix $\L=2$ and $\alpha=1$. In Figure \ref{fig:WN_TestError_LearningRateRatio_Grid_2} we again analyze the error $\epsilon_1$ among different algorithms. We see that in general as $\lrr$ decreases, the error $\epsilon_1$ decreases. In particular, such error is significantly smaller than the error of GD, which is represented by the horizontal line. Note that when $\lrr$ is too large such that the condition of Theorem \ref{theorem:optimality_constant_rate} is violated, WN-GD is no longer guaranteed to outperform GD.
%
\subsection{Effects of Layer}
\label{subsection:Effects of Layer}
%
In this section we compare results for $\L=2$ and $\L=3$. To ensure a fair comparison, we generate initialization in the following way.
Choose an initialization scale $\alpha$ and the vector $\w_0 = \bo/\sqrt{\N}$ so that $\|\w_0\|_2 = 1$. For each $\L$, to make $\|(\g_0\w_0)^{\odot\L}\|_1 = \alpha$ we set 
\begin{equation*}
    \g_0 = \frac{\alpha^{\frac{1}{\L}}}{\|\w_0\|_\L} = \alpha^{\frac{1}{\L}}\N^{\frac{1}{2}-\frac{1}{\L}}.
\end{equation*}

The results are shown in Figure \ref{fig:WN_Layers}. In both cases ($\L=2$ and $\L=3$), we observe significant improvement of reconstruction error with weight normalization. Note that the significant improvement for $\L=3$ requires smaller initialization than the case of $\L=2$. Here the improvement ratio (reconstruction error for GD divided by reconstruction error for WN-GD) at $\alpha = 0.1$ is $173810$ for $\L=2$, and $92164$ for $\L=3$.

\begin{figure}[h]
    \centering
    \includegraphics[width = 0.6\textwidth]{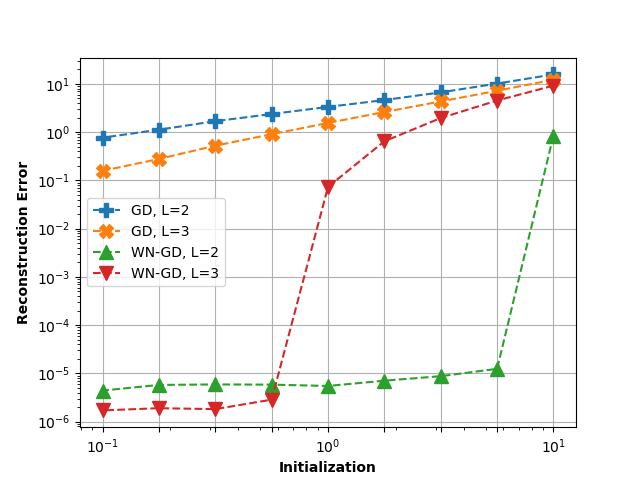}
    \caption{Comparison between $\L=2$ and $\L=3$. WN-GD is better in both cases, but $\L=3$ requires smaller initialization.}
    \label{fig:WN_Layers}
\end{figure}
%
\subsection{Sparse ground truth with positive and negative entries}
\label{subsection:Sparse ground truth with positive and negative entries}
%
\begin{figure}[h]
    \centering
    \begin{subfigure}[b]{0.47\textwidth}
        \centering
        \includegraphics[width = \textwidth]{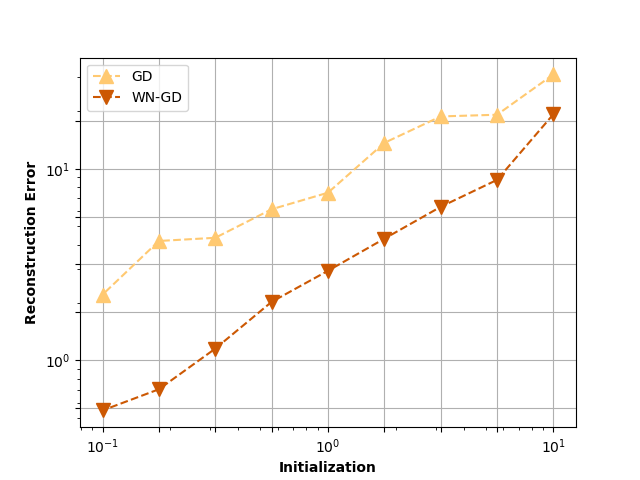}
        \subcaption{WN-GD still achieves smaller reconstruction error, but it is less significant than the positive case.}
        \label{fig:WN_TestError_Grid_Sparse_pm}
    \end{subfigure}
    \hfill
    \begin{subfigure}[b]{0.47\textwidth}
        \centering
        \includegraphics[width = \textwidth]{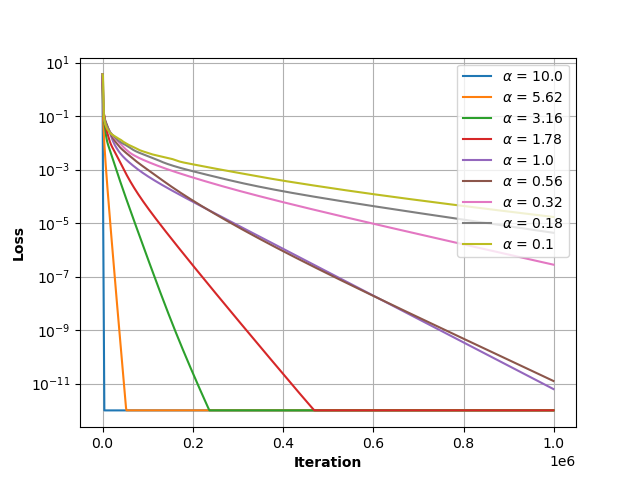}
        \subcaption{Training loss of WN-GD are also decreasing, and the convergence rate seems to related to the initialization.}
        \label{fig:WN_TrainError_Sparse_pm}
    \end{subfigure}
    \caption{The setting is the same as in Figure \ref{fig:exp1}, except that in such setting we can recover ground truth vectors that are not necessarily non-negative.}
    \label{fig:exp4}
\end{figure}
The setting of the forth experiment (Figure \ref{fig:exp4}) is the same as the first experiment, except that the ground truth is not constrained to have positive entries, and the loss function we use here is the same as in \eqref{eq:loss pm}
\begin{equation*}
    \mathcal{L}_\pm(\u,\v) = \frac{1}{2\L}\|\A(\u^{\odot\L}- \v^{\odot\L}) - \y\|_2^2, \quad\u_0,\v_0>0
\end{equation*}
motivated by \cite{chou2021more,Gissin2019Implicit,woodworth2020kernel}. We observe that the benefit of WN still exists, but less strong as the one in the third experiment. 


\section{Summary and Discussion}
\label{sec:Summary}
In this paper we initiate a study of the implicit bias of gradient descent with weight normalization beyond the linear regression setting. In the overparameterized diagonal linear neural network model, we show that weight normalization provably enables a robust implicit regularization towards sparse solutions that holds beyond the regime of small initialization.  We moreover show a linear rate of convergence and an explicit dependence on the initialization scale, indicating that smaller initialization corresponds to slower convergence rate. Numerical experiments are consistent with our theory, and the key quantities such as invariants and the proof strategies can potentially be applied in more general settings.

There are still many remaining questions such as
\begin{enumerate}
    \item Do our results generalize to other settings which currently require small initialization to prove an implicit bias?
    \item Can we use the proof strategy to study the effect of weight normalization on neural networks with first-order homogeneous activation functions, such as ReLU?
    \item Can we extend the results here from gradient flow to gradient descent?
    \item Is there a reasonable choice of time-dependent rate that might outperform constant rate, or does the time-dependent rate necessarily lose the magnified implicit bias as shown in Theorem \ref{theorem:optimality_dynamic_rate}?    
\end{enumerate}


\section*{Acknowledgement}
\label{sec:Acknowledgement}
R. Ward is grateful for support from AFOSR MURI FA9550-19-1-0005, NSF DMS 1952735, NSF HDR 1934932, HDR TRIPODS Phase II grant 2217058, and NSF CCF 2019844 H.C. and H.R. acknowledge funding from the Deutsche Forschungsgemeinschaft (DFG) through the Collaborative Research Center Sparsity and Singular Structures (SFB 1481).


\section*{Data Availability Statement}
\label{sec:Data Availability Statement}

The data underlying this article will be shared on reasonable request to the corresponding author.



\newpage
\appendix

\section{Appendix}
\label{sec:Appedix}

\begin{proof}[Proof of Lemma \ref{lemma:rescaled_learning_rate}]
First note that for
\begin{equation}\label{eq:fix_ratio_all}
    \g^{(\a)} = \g
    \quad\text{and}\quad
    \w^{(\a)} = \a\w
\end{equation}
we obtain
\begin{align*}
    \nabla_{\g^{(\a)}} \LossM(\g^{(\a)},\w^{(\a)})
    &= \nabla_{\g} \LossM(\g,\w),\\
    \nabla_{\w^{(\a)}} \LossM(\g^{(\a)},\w^{(\a)})
    &=\a^{-1}\nabla_{\w} \LossM(\g,\w).
\end{align*}
Consequentially, 
\begin{align}
    \partial_{\t}\g^{(\a)}
    &= -\eta_\g\nabla_{\g^{(\a)}}\LossM(\g^{(\a)},\w^{(\a)})
    = \partial_{\t}\g, \label{eq:fix_ratio_g}\\
    \partial_{\t}\w^{(\a)}
    &= - \a^2\eta_\w\nabla_{\w^{(\a)}}\LossM(\g^{(\a)},\w^{(\a)})
    = \a\partial_{\t}\w \label{eq:fix_ratio_w}.
\end{align}
Since \eqref{eq:fix_ratio_all} holds for $\t=0$ and is preserved due to \eqref{eq:fix_ratio_g} and \eqref{eq:fix_ratio_w}, \eqref{eq:fix_ratio_all} holds for all $\t\geq0$.
\end{proof}
\begin{proof}[Proof of Lemma \ref{lemma:constant_norm}]
A direct computation yields
\begin{align*}
    \partial_{\t}\|\w\|_2^2
    &= 2\w^\tT\partial_{\t}\w
    = -2\w^\tT\nabla_{\w} \LossM(\g,\w)\\
    &= -2\frac{\g}{\|\w\|_2}\w^\tT(\I-\Pw) \nabla\Loss\left(\frac{\g}{\|\w\|_2}\w\right)
    = 0
\end{align*}
because $\w^\tT(\I-\Pw)=0$. This implies the claim.
\end{proof}
\begin{proof}[Proof of Lemma \ref{lemma:non_negative}]
Note that $\nabla_{\g}\LossM(\g,\w)$ is local Lipschitz continuous in $\g$, and $\nabla_{\w} \LossM(\g,\w)$ is local Lipschitz continuous in $\w$. Hence by the Picard–Lindel\"of theorem the trajectory is unique. In particular, $\partial_{\t}\g$ and $\we_\n$ cannot reach zero at finite time, since this would otherwise contradict to the uniqueness of the trajectory (if we apply Picard–Lindel\"of theorem backward in time). We now show that $\g=0$ implies that $\partial_{\t}\g = 0$ and $\we_\n = 0$ implies that $\partial_{\t}\we_\n = 0$. By \eqref{eq:dLdx}, \eqref{eq:dLdg}, and \eqref{eq:dLdw}, we have
\begin{align}
    \partial_{\t}\g
    &= -\eta_\g\nabla_{\g}\LossM(\g,\w)
    = -\frac{\eta_\g}{\|\w\|_2}\w^\tT\nabla\Loss\left(\frac{\g}{\|\w\|_2}\w\right)\\
    \partial_{\t}\w
    &= - \eta_\w\nabla_{\w} \LossM(\g,\w)
    = -\eta_\w\frac{\g}{\|\w\|_2}(\I-\Pw) \nabla\Loss\left(\frac{\g}{\|\w\|_2}\w\right)\nonumber\\
    &= -\eta_\w\frac{\g}{\|\w\|_2}\nabla\Loss\left(\frac{\g}{\|\w\|_2}\w\right)
    +\eta_\w\frac{\g}{\|\w\|_2^3}\w\w^\tT \nabla\Loss\left(\frac{\g}{\|\w\|_2}\w\right)\nonumber\\
    &= -\eta_\w\frac{\g}{\|\w\|_2}\nabla\Loss\left(\frac{\g}{\|\w\|_2}\w\right) -\eta_\w\frac{\g\partial_{\t}\g}{\eta_\g\|\w\|_2^2}\w.\label{eq:dw_dg_1}
\end{align}
If $\g=0$, then $\nabla\Loss = 0$ and hence $\partial_{\t}\g=0$.
If $\we_\n = 0$, then $[\nabla\Loss]_{\n} = 0$ and hence $\partial_{\t}\we_\n =0$. Since $\partial_{\t}\g$ and $\we_\n$ cannot reach zero at finite time, $\g$ and $\we_\n$ cannot reach zero either. By continuity, the signs of $\g$ and $\we_\n$ will stay constant.
\end{proof}
\begin{proof}[Proof of Lemma \ref{lemma:loss_non_increasing}]
By chain rule we have
\begin{align*}
    \partial_{\t}\LossM(\g,\w)
    &= \left\langle \nabla_{\g}\LossM(\g,\w)\,,\,\partial_{\t}\g\right\rangle
    + \left\langle \nabla_{\w}\LossM(\g,\w)\,,\,\partial_{\t}\w\right\rangle\\
    &= -\eta_\g\|\nabla_{\g}\LossM(\g,\w)\|_2^2 - \eta_\w\|\nabla_{\w}\LossM(\g,\w)\|_2^2
    \leq 0.
\end{align*}
This completes the proof.
\end{proof}
\begin{proof}[Proof of Lemma \ref{lemma:Convergence rate}]
By Lemma \ref{lemma:constant_norm}, $\|\w\|_2 = \|\w_0\|_2 = 1$. By Lemma \ref{lemma:loss_non_increasing}, we have
\begin{equation*}
    \partial_{\t}\LossM(\g,\w)
    = -\eta_\g\|\nabla_{\g}\LossM(\g,\w)\|_2^2 - \eta_\w\|\nabla_{\w}\LossM(\g,\w)\|_2^2.
\end{equation*}
Note that by \eqref{eq:dLdg} and \eqref{eq:dLdw},
\begin{align*}
    -\eta_\g\|\nabla_{\g}\LossM(\g,\w)\|_2^2 - \eta_\w\|\nabla_{\w}\LossM(\g,\w)\|_2^2
    &= -\nabla\Loss(\x)^\tT\left(\eta_\g\Pw + \eta_\w \g^2(\I-\Pw)\right)\nabla\Loss(\x)\\
    &\leq -\min\left(\eta_\g,\eta_\w \g^2\right)\|\nabla\Loss(\x)\|_2^2.
\end{align*}
Suppose that there exist constant $\c_\g,\c_\w,\c_\x>0$ such that $\eta_\g\geq\c_\g$, $\eta_\w\geq\c_\w$, and $|\x|\geq \c_\x$ for all $\t\in[\t_0,\T]$, then $\g^2 = \|\x\|_2^2 \geq |I|\cdot\c_\x^2$ and hence
\begin{align*}
    \min\left(\eta_\g,\eta_\w \g^2\right)
    &\geq \min\left(\c_\g,\c_\w |I|\cdot\c_\x^2\right).
\end{align*}
On the other hand,
\begin{align*}
    \left\|\nabla\Loss\left(\x\right)\right\|_2^2
    &\geq \c_\x^{2\L-2} \|\A|_I^\tT(\A\x^{\odot\L} - \y)\|_2^2\\
    &\geq \c_\x^{2\L-2}\sigma_{\min}^2(\A|_I)\|\A\x^{\odot\L} - \y\|_2^2
    = 2\L\c_\x^{2\L-2}\sigma_{\min}^2(\A|_I)\Loss(\x).
\end{align*}
Putting together the estimates, we obtain
\begin{align*}
    \partial_\t\Loss(\x)
    \leq -\min\left(\c_\g,|I|\c_\w \c_\x^2\right)2\L\c_\x^{2\L-2}\sigma_{\min}^2(\A|_I)\Loss(\x).
\end{align*}
By Gronwall's inequality we get the linear convergence rate
\begin{align*}
    \Loss(\x(\t)) \leq \Loss(\x(\t_0))\exp\left(-\min\left(\c_\g,|I|\c_\w \c_\x^2\right)2\L\c_\x^{2\L-2}\sigma_{\min}^2(\A|_I)(\t-\t_0)\right).
\end{align*}
This completes the proof.
\end{proof}
\begin{proof}[Proof of Lemma \ref{lemma:Bounded above implies bounded below}]
    Since $\r$ defined in \eqref{eq:re-scaling factor} is both bounded above and below, it suffices to prove the case for $\x$. We will prove the statement by contradiction. Suppose $\xe_\j$ is not bounded away from zero.  Then, there exists a sequence $\{\t_k\}_{k\in\mathbb{N}}$ such that $\lim_{k\to\infty}\xe_\j(\t_k) = 0$. By assumption there exists $\v>0$ such that $\A\v=0$. Then for any $\z$ such that $(\I-\PA)\z=0$, it holds $$\langle\z,\v\rangle = \langle \PA \z, \v\rangle + \langle (\I - \PA) \z, \v \rangle = \langle \z, \PA \v\rangle = 0.$$

    Let us now consider the case where $\L=2$. By the invariant defined in Lemma \ref{lemma:inv_proj_original}, for all $k\in\mathbb{N}$,
    \begin{equation*}
        (\I-\PA)\log(\x(\t_\k)) = (\I-\PA)\log(\x_0).
    \end{equation*}
    This implies that there exists a sequence $\{\z_k\}_{k\in\mathbb{N}}$ such that $(\I-\PA)\z_k=0$ and
    \begin{equation*}
        \log(\x(\t_k)) = \log(\x_0) + \z_k.
    \end{equation*}
    Taking the inner product with $\v$ on both sides and using the fact that $\langle\z_k,\v\rangle = 0$, we obtain
    \begin{equation*}
        \langle\log(\x(\t_k)),\v\rangle = \langle \log(\x_0),\v\rangle.
    \end{equation*}
    In the limit as $k\to\infty$, the left hand side becomes $-\infty$ because $\x$ is upper bounded while $\log(\xe_\j)$ goes to $-\infty$. This is a contradiction since the right hand side is just a constant. Therefore all entries of $\x$ must be bounded away from zero.

    The same proof strategy works for $\L\neq 2$ as well, leading to 
    \begin{equation*}
        \langle\x^{\odot\frac{2}{\L}-1}(\t_k),\v\rangle = \langle \x_0^{\odot\frac{2}{\L}-1},\v\rangle.
    \end{equation*}
    The left hand side tends to infinity as $k\to\infty$ if $\lim_{k\to\infty}\xe_\j(\t_k) = 0$. This completes the proof.
\end{proof}
\begin{proof}[Proof of Lemma \ref{lemma:unique}]
    Let us first discuss the case where $\L=2$. Since $\PA$ is a projection, by decomposing vectors into the form $\x = \PA\x + (\I-\PA)\x$, we have
    \begin{align*}
        &\lim_{\t\to\infty}\langle \xprod^{(1)}(\t) - \xprod^{(2)}(\t), \log(\xprod^{(1)}(\t))-\log(\xprod^{(2)}(\t))\rangle\\
        &\quad=\lim_{\t\to\infty}\langle \underbrace{\PA[\xprod^{(1)}(\t) - \xprod^{(2)}(\t)]}_{\text{converges to }0}\,,\, \underbrace{\log(\xprod^{(1)}(\t))-\log(\xprod^{(2)}(\t))}_{\text{bounded}}\rangle\\
        &\qquad+\langle\underbrace{\xprod^{(1)}(\t) - \xprod^{(2)}(\t)}_{\text{bounded}}\,,\,\underbrace{(\I-\PA)[\log(\xprod^{(1)}(\t))-\log(\xprod^{(2)}(\t))]}_{\text{converges to }0}\rangle\\
        &\quad=0
    \end{align*}
    Note that due to the monotonicity of $\log$, we have $(a-b)(\log(a)-\log(b))\geq 0$, and consequently
    \begin{equation*}
        \boldsymbol\xi(\t):=[\xprod^{(1)}(\t) - \xprod^{(2)}(\t)]\odot[\log(\xprod^{(1)}(\t))-\log(\xprod^{(2)}(\t))] \geq 0.
    \end{equation*}
    A more compact way to express $\boldsymbol\xi$ is via the difference
    \begin{equation*}
        \boldsymbol\Delta:=|\xprod^{(1)} - \xprod^{(2)}|
    \end{equation*}
    and the expression
    \begin{align}
        \boldsymbol\xi
        &= \boldsymbol\Delta \odot \log(\bo + \boldsymbol\Delta\odot \min(\xprod^{(1)},\xprod^{(2)})^{\odot-1})\nonumber\\
        &\geq \boldsymbol\Delta \odot \log(\bo + \boldsymbol\Delta\odot (\xprod^{(1)})^{\odot-1})\label{eq:xi_delta_1}.
    \end{align}
    Since $\boldsymbol\xi\geq 0$ and $\langle\boldsymbol\xi(\t),\bo\rangle$ converges to zero as $\t\to\infty$, we can deduce that $\boldsymbol\xi$ converges to zero as $\t\to\infty$. Together with \eqref{eq:xi_delta_1} and the assumption that $\xprod^{(1)}$ is uniformly bounded above, we conclude that $\boldsymbol\Delta$ must also converges to zero. Since $\boldsymbol\Delta$ converges to zero, our conclusion follows.
    
    For $\L>2$ the analysis is similar. First use the assumption to deduce that
    \begin{align*}
        0=\lim_{\t\to\infty}\langle \xprod^{(1)}(\t) - \xprod^{(2)}(\t), (\xprod^{(2)}(\t))^{\odot \frac{2}{\L}-1}-(\xprod^{(1)}(\t))^{\odot \frac{2}{\L}-1}\rangle.
    \end{align*}
    Since the vector
    \begin{equation*}
        \boldsymbol\xi(\t) := [\xprod^{(1)}(\t) - \xprod^{(2)}(\t)]\odot[(\xprod^{(2)})^{\odot \frac{2}{\L}-1}(\t)-(\xprod^{(1)})^{\odot \frac{2}{\L}-1}(\t)]
    \end{equation*}
    is non-negative and $\langle\boldsymbol\xi(\t),\bo\rangle$ converges to zero, we can deduce that $\boldsymbol\xi$ converges to zero. We will use the following fact: if $a=b+\delta$ with $a,b,\delta\geq 0$, then
    \begin{align*}
        \frac{1}{b} - \frac{1}{a} \geq
        \begin{cases}
            \frac{1}{2b} &\text{if }\delta\geq b\\
            \frac{\delta}{2b^2} &\text{if }\delta\leq b
        \end{cases}
    \end{align*}
    and consequently
    \begin{equation*}
        \frac{1}{b} - \frac{1}{a} \geq \frac{1}{2b^2}\cdot\min\left(b,\delta\right).
    \end{equation*}
    By substituting ${\bf a}= \max(\xprod^{(1)},\xprod^{(2)})^{\odot 1-\frac{2}{\L}}$ and ${\bf b}= \min(\xprod^{(1)},\xprod^{(2)})^{\odot 1-\frac{2}{\L}}$ into $\boldsymbol\xi$, we obtain that
    \begin{align*}
        \boldsymbol\xi
        &= |\xprod^{(1)} - \xprod^{(2)}|\odot[{\bf b}^{\odot -1}-{\bf a}^{\odot -1}]\\
        &\geq \frac{1}{2}|\xprod^{(1)} - \xprod^{(2)}|\odot\min\left({\bf b}^{\odot -1},|(\xprod^{(1)})^{\odot 1-\frac{2}{\L}} - (\xprod^{(2)})^{\odot 1-\frac{2}{\L}}|\odot{\bf b}^{\odot -2}\right)\\
        &\geq \frac{1}{2}|\xprod^{(1)} - \xprod^{(2)}|\odot\min\left((\xprod^{(1)})^{\odot \frac{2}{\L}-1}, |(\xprod^{(1)})^{\odot 1-\frac{2}{\L}} - (\xprod^{(2)})^{\odot 1-\frac{2}{\L}}|\odot(\xprod^{(1)})^{\odot \frac{4}{\L}-2}\right)
    \end{align*}
    We again see that the difference $|\xprod^{(1)}(\t) - \xprod^{(2)}(\t)|$ converges to zero because $\boldsymbol\xi(\t)$ converges to zero and $\xprod^{(1)}, \xprod^{(2)}$ are uniformly bounded above and below. Hence the proof is complete.
\end{proof}
\begin{proof}[Proof of Lemma \ref{lemma:boundedness}]
    Denote $\gprod = \exp(\g^2/(2\lrr))$. Without loss of generality we may assume $\|\w_0\|_2=1$. By Lemma \ref{lemma:constant_norm}, $\|\w(\t)\|_2 = 1$ for all $\t\geq 0$. According to Lemma \ref{lemma:loss_non_increasing}, the loss is non-increasing, and hence $\|\A(\g\w)^{\odot\L}-\b\|_\infty$ is upper bounded. Consequently,
    \begin{equation}\label{eq:PA_bound}
        \|\PA(\g\w)^{\odot\L}\|_\infty \leq B_1
    \end{equation}
    for some $B_1\geq 0$. On the other hand, by Lemma \ref{lemma:inv_proj_constant_rate}, the quantity
    \begin{equation}\label{eq:I-PA_bound}
        (\I-\PA)\cdot\begin{cases}
        \log(\gprod\w)&\text{if }\L=2\\
        (\gprod\w)^{\odot 2-\L}&\text{if }\L\neq 2
        \end{cases}
    \end{equation}
    equals its value at initialization for all $\t\geq 0$. Thus its $\ell_\infty$-norm is upper bounded by some $B_2\geq0$.

    Let us first study the case where $\L=1$ by examining the inner product $\langle\g\w,\gprod\w\rangle$. By decomposing $\I=\PA + (\I-\PA)$, we obtain 
    \begin{align*}
        \left\langle\g\w,\gprod\w\right\rangle
        &= \left\langle\g\w,\PA\gprod\w\right\rangle + \left\langle\g\w,(\I-\PA)\gprod\w\right\rangle\\
        &= \left\langle\PA\g\w,\gprod\w\right\rangle + \left\langle\g\w,(\I-\PA)\gprod\w\right\rangle\\
        &\leq B_1 \left\|\gprod\w\right\|_1 + B_2 \|\g\w\|_1\\
        &= \left(B_1\gprod+B_2\g\right) \left\|\w\right\|_1\\
        &\leq \left(B_1\gprod+B_2\g\right)\sqrt{\N},
    \end{align*}
    where the last inequality comes from the fact that $\|\w\|_1\leq\sqrt{\N}\|\w\|_2  = \sqrt{\N}$. Since the left hand side can be explicitly expressed as
    \begin{equation*}
        \left\langle\g\w,\gprod\w\right\rangle
        = \g\gprod\|\w\|_2^2 = \g\gprod,
    \end{equation*}
    we have
    \begin{equation*}
        \g\exp\left(\frac{\g^2}{2\lrr}\right) \leq \left(B_1\exp\left(\frac{\g^2}{2\lrr}\right)+B_2\g\right) \sqrt{\N}.
    \end{equation*}
    This implies that $\g$ is uniformly bounded, because otherwise the left hand side will eventually exceed the right hand side.
    
    For $\L=2$ we use a similar strategy by examining the inner product $\langle(\g\w)^{\odot 2},\log(\gprod\w)\rangle$. We first obtain a lower bound of the inner product,
    \begin{align*}
        \langle(\g\w)^{\odot 2},\log(\gprod\w)\rangle
        &= \g^2(\log(\gprod)\langle\w^{\odot 2},\bo\rangle + \langle\w^{\odot 2},\log(\w)\rangle)\\
        &=\g^2\left(\log(\gprod)\|\w\|_2^2 + \langle\w^{\odot 2},\log(\w)\rangle\right).
    \end{align*}
    Since $\xi^2\log(\xi)\geq -\frac{1}{2e}$ for all $\xi\geq0$, we have
    \begin{equation}\label{eq:inner_upperbound_L2}
        \langle(\g\w)^{\odot 2},\log(\gprod\w)\rangle
        \geq \g^2\left(\log(\gprod) - \frac{\N}{2e}\right).
    \end{equation}
    We now derive an upper bound. By decomposing $\I=\PA + (\I-\PA)$, we obtain 
    \begin{align*}
        \left\langle(\g\w)^{\odot 2},\log\left(\gprod\w\right)\right\rangle
        &= \left\langle(\g\w)^{\odot 2},\PA\log\left(\gprod\w\right)\right\rangle + \left\langle(\g\w)^{\odot 2},(\I-\PA)\log\left(\gprod\w\right)\right\rangle\\
        &= \left\langle\PA(\g\w)^{\odot 2},\PA\log\left(\gprod\w\right)\right\rangle + \left\langle(\g\w)^{\odot 2},(\I-\PA)\log\left(\gprod\w\right)\right\rangle\\
        &\leq B_1 \left\|\PA\log\left(\gprod\w\right)\right\|_1 + B_2 \|(\g\w)^{\odot 2}\|_1\\
        &= B_1 \left\|\log(\gprod)\cdot\PA\bo+\PA\log\left(\w\right)\right\|_1 + B_2\g^2\|\w\|_2^2\\
        &\leq B_1\log(\gprod)\|\PA\bo\|_1 + B_1\left\|\PA\log\left(\w\right)\right\|_1 + B_2\g^2.
    \end{align*}
    Note that $B_1\left\|\PA\log\left(\w\right)\right\|_1$ is upper bounded by some constant $C$ because the entries of $\w$ are bounded both above and below. Combining this with \eqref{eq:inner_upperbound_L2}, we have
    \begin{equation*}
        \frac{\g^4}{2\lrr} - \frac{e\N\g^2}{2} \leq \left(\frac{B_1}{2\lrr}\|\PA\bo\|_1+B_2\right) \g^2 + C,
    \end{equation*}
    which implies that $\g$ is upper bounded because the left hand side scales like $\g^4$ while the right hand side scales like $\g^2$.
\end{proof}
\begin{proof}[Proof of Lemma \ref{lemma:dynamic_reduction}]
By \eqref{eq:dLdg} and \eqref{eq:dLdw},
\begin{align*}
    \partial_{\t}\left(\frac{\g}{\|\w\|_2}\w\right)
    &= \w\frac{\partial_{\t}\g}{\|\w\|_2} + \frac{\g}{\|\w\|_2}(\I-\Pw)\partial_{\t}\w\\
    &= -\eta_\g\Pw\nabla\Loss\left(\frac{\g}{\|\w\|_2}\w\right)
    - \eta_\w\frac{\g^2}{\|\w\|_2^2}\left(\I-\Pw\right)^2\nabla\Loss\left(\frac{\g}{\|\w\|_2}\w\right)\\
    &= -\left(\eta_\g\Pw+ \frac{\eta_\w\g^2}{\|\w\|_2^2}(\I-\Pw)\right)\nabla\Loss\left(\frac{\g}{\|\w\|_2}\w\right).
\end{align*}
Since $\|\w\|_2=1$, by Lemma \ref{lemma:constant_norm} and $\g = \|\x\|_2$, we obtain
\begin{align*}
    \Pw = \frac{\w\w^\tT}{\|\w\|_2^2} = \frac{\g\w(\g\w)^\tT}{\g^2} = \frac{\x\x^\tT}{\|\x\|_2^2} = \Px
\end{align*}
and hence \eqref{eq:dynamic_reduction} holds.
\end{proof}
\begin{proof}[Proof of Lemma \ref{lemma:convergence_dynamic_rate}]
Denote $\xprod = \x^{\odot\L}$ and $\x_0=\x(0)$.
We prove the statement by contradiction. Suppose that $\Loss(\x(\t))$ does not converge to zero. Since $\Loss$ is non-increasing in $\t$ according to Lemma \ref{lemma:loss_non_increasing}, $\Loss$ is bounded away from zero. Therefore, there exists $\epsilon>0$ such that $\Loss(\x(\t))\geq \epsilon$ for all $\t\geq 0$.

Let $\z\in S_{+}$. Since $\g(0)>0$ and $\w(0)\geq0$, we have $\x\geq0$ for all $\t\geq0$ by Lemma \ref{lemma:non_negative}. Hence $D_{\F}(\z,\xprod(\t))$ is well-defined. By Lemma \ref{lemma:Bregman_divergence_non_increasing}, $D_{\F}(\z,\xprod(\t))$ is non-increasing in $\t$ and hence bounded above by $D_{\F}(\z,\xprod_0)$. By Lemma \ref{lemma:Bregman_divergence}, $D_{\F}$ is non-negative and hence bounded below. Therefore by \eqref{eq:Bregman_divergence_non_increasing},
\begin{align}\label{eq:Bregman_geq_x}
    D_{\F}(\z,\xprod_0)
    \geq -\int_{0}^{\infty} \partial_{\t}D_{\F}(\z,\xprod(\t)) d\t
    = \int_{0}^{\infty} 2\L\|\x(\t)\|_2^2\cdot\Loss(\x(\t)) d\t
    \geq 2\L\epsilon\int_{0}^{\infty} \|\x(\t)\|_2^2 d\t.
\end{align}
Since $D_{\F}(\z,\xprod_0)<\infty$, by \eqref{eq:Bregman_geq_x} $\|\x(\t)\|_2$ cannot be bounded away from zero. This implies that there exists an increasing sequence $\{\t_k\}_{k\in[\mathbb{N}]}$ such that $\lim_{k\to\infty}\|\x(\t_k)\|_2 = 0$. Together with the fact that $\x$ is non-negative, we get that for all $\n\in[\N]$, $\lim_{k\to\infty}\xe_\n(\t_k) = 0$, and consequently $\lim_{k\to\infty}\xprode_\n(\t_k) = 0$.

We now prove that $D_{\F}$ eventually exceeds its initial value, i.e. there exists $\T>0$ such that $D_{\F}(\z,\xprod(\T)) > D_{\F}(\z,\xprod_0)$, and hence the contradiction. Note that the Bregman divergence in our case is given by
\begin{equation*}
    D_{\F}(\z,\xprod)
    =\begin{cases}
    \frac{1}{2} \langle \z \odot \log(\z) - \z + \xprod , \1 \rangle  - \frac{1}{2}\langle \log(\xprod),\z\rangle & \text{if }\L = 2,\\
     \langle \frac{\L}{2(2-\L)}\z^{\odot \frac{2}{\L}}+\frac{1}{2}\xprod^{\odot \frac{2}{\L}}, \1 \rangle - \frac{1}{2-\L}\langle \xprod^{\odot \frac{2}{\L}-1},\z\rangle & \text{if }\L > 2.
    \end{cases}
\end{equation*}
As $\x$ converges to zero, all terms except the last one converge to zero, while the last term blows up to infinity because $\z\geq0$ and $\z$ is not identically zero (as a consequence of $\y$ being not identically zero). Thus there exists $k\in\mathbb{N}$ such that $D_{\F}(\z,\xprod(\t_k)) > D_{\F}(\z,\xprod_0)$. This contradicts Lemma \ref{lemma:Bregman_divergence_non_increasing}. Therefore $\Loss(\x(\t))$ converges to zero.
\end{proof}


\end{document}